\documentclass[twoside,leqno,twocolumn]{article}

\usepackage[letterpaper]{geometry}

\usepackage{ltexpprt}

\usepackage{cite}
\usepackage{amsmath, amssymb, amsfonts}
\usepackage{graphicx}
\graphicspath{ {./images/} }
\usepackage{textcomp}
\usepackage{xcolor}
\usepackage{stfloats}
\usepackage{subfigure}
\usepackage{hyperref}
\usepackage{makecell}
\usepackage{multirow}
\usepackage{algpseudocode}
\usepackage[english]{babel}
\usepackage{algorithm}
\usepackage{xr,xfrac}
\usepackage{enumitem}

\usepackage[font=small,labelfont=bf]{caption}
\setitemize{noitemsep,topsep=0pt,parsep=0pt,partopsep=0pt}
\setenumerate{noitemsep,topsep=0pt,parsep=0pt,partopsep=0pt}

\makeatletter
\newcommand*{\addFileDependency}[1]{
  \typeout{(#1)}
  \@addtofilelist{#1}
  \IfFileExists{#1}{}{\typeout{No file #1.}}
}

\makeatother

\newcommand*{\myexternaldocument}[1]{%
    \externaldocument{#1}%
    \addFileDependency{#1.tex}%
    \addFileDependency{#1.aux}%
}

\myexternaldocument{technical}

\DeclareMathOperator\supp{supp}
\newcommand{\reals}{\mathbb{R}}
\newcommand{\naturals}{\mathbb{N}}
\newcommand{\vc}[1]{\mathbf{#1}}
\newcommand{\domain}{\mathcal{M}}
\newcommand{\iset}[2]{\mathcal{#1}_{#2}}
\newcommand{\argmax}{\mathop{\arg\,\max}}

\newcommand{\cvx}{\texttt{conv}}

\newtheorem{assumption}{Assumption}

\newcommand{\fullversion}[2]{#2}

\AtBeginDocument{%
  \providecommand\BibTeX{{%
    \normalfont B\kern-0.5em{\scshape i\kern-0.25em b}\kern-0.8em\TeX}}}

\begin{document}

\makeatletter
\def\thanks#1{\protected@xdef\@thanks{\@thanks
        \protect\footnotetext{#1}}}
\makeatother

\title{\Large Submodular Maximization via Taylor Series Approximation\thanks{Supported by NSF grant CCF-1750539.}}
\author{Gözde Özcan$^{1}$
\and Armin Moharrer$^{1}$
\and Stratis Ioannidis$^{1}$
\thanks{\scriptsize{$^1$\{gozcan, amoharrer, ioannidis\}@ece.neu.edu, Electrical and Computer Engineering Department, Northeastern University, Boston, MA, USA.}}}


\date{}

\maketitle




\begin{abstract} \small\baselineskip=9pt We study submodular maximization problems with matroid constraints, in particular, problems where the objective  can be expressed via compositions of analytic  and multilinear functions. We show that for functions of this  form, the so-called \emph{continuous greedy} algorithm \cite{calinescu2011maximizing} attains a ratio arbitrarily close to $(1-1/e) \approx 0.63$ using a deterministic estimation via Taylor series approximation. This drastically reduces execution time over prior art that uses sampling.
\end{abstract}
\section{Introduction.}\label{sec:intro}
Submodular functions are set functions that exhibit a diminishing returns property. They naturally arise in many applications, including data summarization \cite{lin2011class,lin2010multi,gygli2015video}, facility location \cite{krause2014submodular}, recommendation systems \cite{mirzasoleiman2016fast}, influence maximization \cite{kempe2003maximizing}, sensor placement \cite{krause2008near}, dictionary learning \cite{jiang2012submodular,zhu2014cross}, and active learning \cite{badanidiyuru2014streaming}. In these problems, the goal is to maximize a submodular function subject to matroid constraints. These problems are in general NP-hard, but a celebrated greedy algorithm \cite{nemhauser1978best} achieves a $1-1/e$ approximation ratio on uniform matroids. Unfortunately, for  general matroids the approximation ratio drops to $1/2$ \cite{nemhauser1978analysis}. 

The \emph{continuous greedy} algorithm \cite{vondrak2008optimal, calinescu2011maximizing} improves this bound. The algorithm maximizes the \emph{multilinear relaxation} of a submodular function in the continuous domain, guaranteeing a $1-1/e$ approximation ratio \cite{calinescu2011maximizing}. The fractional solution  is then rounded to a feasible integral solution (without compromising the objective value), e.g., via pipage rounding \cite{ageev2004pipage} or swap rounding \cite{chekuri2010dependent}. The multilinear relaxation of a submodular function is its expected value under independent Bernoulli trials; however, computing this expectation is hard in general. The state of the art is   to estimate the multilinear relaxation via sampling \cite{calinescu2011maximizing, vondrak2008optimal}. Nonetheless, the number of samples required in order to achieve the superior $1-1/e$ guarantee is quite high; precisely because of this, the resulting running time of continuous greedy is $O(N^8)$ in input size $N$ \cite{calinescu2011maximizing}. 

Nevertheless, for some submodular functions, the multilinear relaxation can be computed efficiently. One well-known example is the \emph{coverage function}, which we  describe in Sec.~\ref{sec:multilin}; given subsets of a ground set, 
the coverage function computes the number of elements covered in the union of these subsets.  The multilinear relaxation for coverage can be computed precisely, without sampling, in polynomial time. 
This is well-known, and has been exploited in several different contexts \cite{karimi2017stochastic,singer2012win,ageev2004pipage}. 

We extend the range of problems for which the multilinear relaxation can be computed efficiently. First, we observe that this property naturally extends to  \emph{multilinear functions}, a class that  includes  coverage functions. We then consider a class of submodular objectives that are a summation over non-linear functions of these multilinear functions.  Our key observation is that  the polynomial expansions of  these  functions are again multilinear; hence, compositions of  multilinear functions with arbitrary \emph{analytic} functions, that can be approximated by a Taylor series, can be computed efficiently. A broad range of problems, e.g., data summarization, influence maximization, facility location, and  cache networks (c.f. Sec.~\ref{sec:examples}), can be expressed in this manner and solved efficiently via our approach. 

In summary, we make the following contributions:
\begin{itemize}
\item We introduce a class of submodular functions that can be expressed as weighted compositions of analytic  and multilinear functions. 
    \item We propose a novel polynomial series estimator  for approximating the multilinear relaxation of this class of problems.  
    \item We provide strict theoretical guarantees for a variant of the continuous greedy algorithm that uses our estimator. We show that the sub-optimality due to our polynomial expansion is bounded by a quantity that can be  made arbitrarily small by increasing the polynomial order. 
    \item We show that multiple applications, e.g., data summarization, influence maximization, facility location, and cache networks can be cast as  instances of our framework.
    \item We conduct numerical experiments for multiple problem instances on both synthetic and real datasets. We observe that our estimator achieves $74\%$ lower error, in $89\%$ less time, in comparison with the sampling estimator.   
\end{itemize}
The remainder of the paper is organized as follows. We review related work and technical background in Sections~\ref{sec:related} and~\ref{sec:tech}, respectively. We introduce multilinear functions in Sec.~\ref{sec:multilin}. We  present our estimator and main results in Sec.~\ref{sec:mainres},  examples of cases that can be instances of our problem in Sec.~\ref{sec:examples}, and our numerical evaluation in Sec.~\ref{sec:experiments}. We conclude in Sec.~\ref{sec:conclusion}.
\section{Related Work.} \label{sec:related}
We refer the reader to  Krause and Golovin \cite{krause2014submodular} for a thorough review of submodularity and its applications. 
 
\noindent\textbf{Accelerating Greedy.} The seminal greedy algorithm proposed by Nemhauser et al.~\cite{nemhauser1978best} provides a $1-1/e$ approximation ratio for submodular maximization problems subject to the uniform matroids. However, for general matroids this approximation ratio deteriorates to 1/2 \cite{nemhauser1978analysis}. Several works have introduced variants to  greedy algorithm to accelerate it \cite{minoux1978accelerated, kumar2015fast, mirzasoleiman2015lazier}, particularly for influence maximization \cite{borgs2014maximizing, tang2015influence}. However, these accelerations do not readily apply to the continuous greedy algorithm. 
 
\noindent \textbf{Multilinear Relaxation.} The continuous greedy algorithm was proposed by Vondr{\'a}k \cite{vondrak2008optimal} and Calinescu  et al. \cite{calinescu2011maximizing}.   Maximizing  the multilinear relaxation of  submodular functions improves the 1/2 approximation ratio of the greedy algorithm \cite{nemhauser1978analysis}  to $1-1/e$ \cite{calinescu2011maximizing} over general matroids. Beyond maximization over matroid constraints, the multilinear relaxation has been used to obtain guarantees for non-monotone submodular maximization \cite{feldman2011unified,chekuri2014submodular}, as well as in pipage rounding \cite{ageev2004pipage}. All of these approaches resort to sampling; as we provide general approximation guarantees, our approach can be used to accelerate these algorithms as well.

\noindent\textbf{DR-Submodularity.} Submodular functions have also been studied in the continuous domain recently. Continuous functions that exhibit the diminishing returns property are termed \emph{DR-submodular} functions \cite{bian2017continuous, bian2017guaranteed, chekuri2015multiplicative, bach2019submodular, niazadeh2018optimal, soma2017non}, and arise in mean field inference \cite{bian2019optimal}, budget allocation \cite{staib2017robust}, and non-negative quadratic programming \cite{bian2017guaranteed,skutella2001convex}. 
DR-submodular functions are in general neither convex nor concave; however, gradient-based methods \cite{bian2017continuous, bian2017guaranteed, hassani2017gradient, chekuri2015multiplicative} provide  constant approximation guarantees. The multilinear relaxation is also a DR-submodular function; hence,  obtaining fractional solutions to multilinear relaxation maximization problems, without  rounding,  is of independent interest. Our work can thus be used to accelerate precisely this process. 

\noindent \textbf{Stochastic Submodular Maximization.}  Stochastic submodular maximization, in which the objective is itself random, has attracted great interest recently \cite{mokhtari2018conditional, mokhtari2020stochastic, karimi2017stochastic, hassani2017gradient, asadpour2008stochastic}, both in the discrete and continuous domains. A quintessential example is influence maximization \cite{kempe2003maximizing}, where the total number of influenced nodes is determined by random influence models. In short, when submodular or DR-submodular objectives are expressed as expectations,   sampling in gradient-based methods has two sources of randomness (one for sampling the objective, and one for estimating the multilinear relaxation/sampling inputs); continuous greedy still comes with guarantees. 
Our work is orthogonal, in that it can be used to eliminate the second source of randomness. It can therefore be used in conjunction with stochastic methods whenever our assumptions apply.

\noindent\textbf{Connection to Other Works.} 
Our work is closest to, and inspired by, Mahdian et al.~\cite{mahdian2020kelly} and Karimi et al.~ \cite{karimi2017stochastic}. To the best of our knowledge, the only other work that approximates the multilinear relaxation via a power series is~\cite{mahdian2020kelly}. The authors apply this technique to a submodular maximization problem motivated by cache networks. 
We depart by (a) extending this approach to more general submodular functions, (b) establishing formal assumptions under which this generalization yields approximation guarantees, and (c) improving upon earlier guarantees for cache networks by \cite{mahdian2020kelly}. In particular, the authors assume that derivatives are bounded; we relax this assumption, that does not hold  for any of the problems we study here.

 Karimi et al.~\cite{karimi2017stochastic} 
 maximize stochastic \emph{coverage functions} subject to matroid constraints, showing that many different problems can be cast in this setting. 
 Some of the examples we consider (see Sec.~\ref{sec:examples}) consist of compositions of analytic, non-linear functions   with coverage functions; hence, our work can be seen as a direct generalization of \cite{karimi2017stochastic}. 

\section{Technical Preliminaries.}\label{sec:tech}
\subsection{Submodularity and Matroids.}\label{sec:submat}

Given a ground set $V=\{1,\ldots,N\}$ 
of $N$ elements, a set function $f:2^V\rightarrow\reals_+$ is submodular if and only if $f(B \cup \{e\}) - f(B) \leq f(A\cup \{e\}) - f(A)$, for all $A\subseteq B\subseteq V$ and $e\in V$. Function $f$ is  \emph{monotone} if $f(A)\leq f(B)$, for every $A\subseteq B$.

\noindent \textbf{Matroids.} Given a ground set $V$, a matroid is a pair $\mathcal{M}=(V, \mathcal{I})$, where $\mathcal{I}\subseteq 2^V$ is a collection of \emph{independent sets}, for which the following holds: 
\begin{enumerate}
    \item If $B\in \mathcal{I}$ and $A \subset B$, then $A \in \mathcal{I}$.
    \item If $A, B\in \mathcal{I}$ and $|A|< |B|,$ there exists $x \in B\setminus A$ s.t. $A\cup\{x\}\in \mathcal{I}$.
\end{enumerate}
 The \emph{rank} of a matroid $r_{\mathcal{M}}(V)$ is the largest cardinality of its elements, i.e.:
  $  r_{\mathcal{M}}(V) = \max\{|A|: {A}\in\mathcal{I}\}.$
We introduce two examples of matroids:
\begin{enumerate}
    \item \textbf{Uniform Matroids.} The uniform matroid with cardinality $k$ is $\mathcal{I}=\{S\subseteq V, \, |S|\leq k\}$.
    \item \textbf{Partition Matroids.} Let $\mathcal{B}_1,\ldots, \mathcal{B}_m\subseteq V$ be a partitioning of $V$, i.e., $ \bigcap_{\ell=1}^m\mathcal{B}_\ell =\emptyset$ and $\bigcup_{\ell=1}^m\mathcal{B}_\ell = V$. Let also $k_\ell\in \mathbb{N}, \ell=1,\ldots,m$, be a set of cardinalities.   A partition matroid is defined as $\mathcal{I}=\{S\subseteq 2^V \, \mid  \, |S\cap \mathcal{B}_\ell|\leq k_l, \text{ for all } \ell=1,\ldots, m\}.$  
\end{enumerate}

\noindent \textbf{Change of Variables.} 
There is a one-to-one correspondence between a binary vector $\vc{x}\in \{0,1\}^{N}$ and its support $S=\texttt{supp}(\vc{x})$. Hence, a set function $f: 2^V \rightarrow \reals_+$ can be interpreted as $f: \{0,1\}^N \rightarrow \reals_+$ via: 
$f(\vc{x}) \triangleq f(\texttt{supp}(\vc{x}))$ for $\vc{x} \in \{0,1\}^N$. We adopt this convention for the remainder of the paper. 
We also treat matroids as subsets of $\{0,1\}^N$, defined consistently with this change of variables via \begin{align}\mathcal{M}=\{\vc{x}\in\{0,1\}^N: \supp(\vc{x})\in \mathcal{I}\}.\end{align} For example, a partition matroid is: 
\begin{align} \label{eq:part_mat}
\mathcal{M} = \textstyle\left\{\vc{x} \in \{0,1\}^N\,\mid \bigcap_{\ell=1}^m  \left(\sum_{i\in B_\ell} x_i\leq k_\ell\right)\right\}. 
\end{align} 
The \emph{matroid polytope} $P(\mathcal{M})\subseteq [0,1]^{N}$ is the convex hull of matroid $\mathcal{M}$, i.e.,
$P(\mathcal{M}) = \cvx(\mathcal{M}).$
\subsection{Submodular Maximization  Subject to Matroid Constraints.}
We consider the problem of  maximizing a submodular  function $f:\{0,1\}^N\to\reals_+$ subject to matroid constraints $\mathcal{M}$:
\begin{align}\label{eq:subMAX}
 \textstyle\max_{\vc{x} \in \mathcal{M}} f(\vc{x}).
\end{align}
As mentioned in the introduction, the classic greedy algorithm achieves a 1/2 approximation ratio over general matroids, while the continuous greedy algorithm \cite{calinescu2011maximizing} achieves a $1-1/e$ approximation ratio. 
We review the continuous greedy algorithm below.
\subsection{Continuous Greedy Algorithm.}\label{sec:CG}
The multilinear relaxation  of a submodular function $f$ is
the expectation of $f$, assuming inputs $x_i$ are independent Bernoulli random variables, i.e., $G:[0,1]^{N}\rightarrow \reals_+$, and  
\begin{align} \label{eq:multilinear}
\begin{split}
G(\vc{y}) &\!=\! \mathbb{E}_{\vc{x} \sim \vc{y}}[f(\vc{x})]\!=\!\!\!\!\!\!\sum_{\vc{x}\in\{0,1\}^N}\!\!\!\!\!f(\vc{x})\!\!\prod_{i: x_i=1}\!\! y_i \!\!\prod_{i: x_i=0}\!\!(1-y_i),\!\!\!\!\!
\end{split}
\end{align}
where $\vc{y}=[y_i]_{i=1}^N\in [0,1]^N$ is the vector of probabilities $y_i=\mathbb{P}[x_i=1]$.
The continuous greedy  algorithm first maximizes $G$ in the continuous domain, producing an approximate solution to:
\begin{align}\label{eq: multilinProb}
    \textstyle\max_{\vc{y}\in P(\mathcal{M})} G(\vc{y}).
\end{align}
The algorithm initially starts with  $\vc{y}_0=\mathbf{0}$. Then, it proceeds in iterations, where in the $k$-th iteration, it finds a feasible point $\vc{m}_k\in P(\mathcal{M})$ which is a solution for the following linear program: 
\begin{equation} \label{eq:m_k}
    \textstyle\max_{\vc{m} \in P(\mathcal{M})} \big \langle \vc{m}, \nabla G(\vc{y}_k) \big \rangle,
\end{equation}
After finding $\vc{m}_k$, the algorithm updates the current solution $\vc{y}$ as follows:
\begin{equation} \label{eq:y_k}
    \vc{y}_{k+1} = \vc{y}_k + \gamma_k \vc{m}_k,
\end{equation}
where $\gamma_k \in [0, 1]$ is a step size. We summarize the continuous greedy algorithm in Alg.~\ref{alg:cont-greed}.

The output of Alg.~\ref{alg:cont-greed} is within a $1 - 1/e$ factor from the optimal solution $\mathbf{y}^*\in P(\mathcal{M})$ to \eqref{eq: multilinProb} (see Thm.~\ref{thm: sampler_approx} below). This fractional solution can be rounded to produce a solution to (\ref{eq:subMAX}) with the same approximation guarantee using, e.g., either the pipage rounding \cite{ageev2004pipage} or the swap rounding \cite{calinescu2011maximizing,chekuri2010dependent} methods. Both are reviewed in detail in \fullversion{\cite{ozcan2021submodular}}{App.~\ref{sec:rounding}}. 

\noindent\textbf{Sample Estimator.}
The gradient $\nabla G$ is needed to perform step \eqref{eq:m_k}; computing it directly via \eqref{eq:multilinear}, involves a summation over $2^N$ terms. Instead, Calinescu et al.~\cite{calinescu2011maximizing} estimate it via sampling. First, observe that function $G$ is affine w.r.t a coordinate $y_i$. As a result, 
\begin{equation} \label{eq: partial_derivative}
    ({\partial G(\vc{y})}/{\partial y_i}) = \mathbb{E}_{\vc{x} \sim \vc{y}}[f\left([\vc{x}]_{+i}\right)] - \mathbb{E}_{\vc{x} \sim \vc{y}}[f\left([\vc{x}]_{-i}\right)], 
\end{equation}
where $[\vc{x}]_{+i}$ and $[\vc{x}]_{-i}$ are equal to the vector $\vc{x}$ with the $i$-th coordinate set to $1$ and $0$, respectively.
The gradient of $G$ can thus be estimated by (a) producing $T$ random samples $\vc{x}^{(l)}$, for $l \in \{1, \ldots, T\}$ of the random vector $\vc{x}$, consisting of independent Bernoulli coordinates with $\mathbf{P}(x_i = 1) = y_i$, and (b) computing the empirical mean of the r.h.s. of \eqref{eq: partial_derivative},  yielding:
\begin{equation} \label{eq: sampler estimator}
    \widehat{\frac{\partial G(\vc{y})}{\partial y_i}} = \frac{1}{T} \sum\limits_{l=1}^T (f([\vc{x}^{(l)}]_{+i}) - f([\vc{x}^{(l)}]_{-i})).
\end{equation}
This estimator  yields the following guarantee:
\begin{theorem} \label{thm: sampler_approx}
[Calinescu et al. \cite{calinescu2011maximizing}] Consider Algorithm~\ref{alg:cont-greed}, with $\nabla G(\vc{y}_k)$ replaced by $\widehat{\nabla G}(\vc{y}_k)$ given by (\ref{eq: sampler estimator}). Set $T = \frac{10}{\delta^2} (1 + \ln{|V|})$, where $\delta = \frac{1}{40d^2|V|}$ and $d = r_{\mathcal{M}}(V)$ is the rank of the matroid. The algorithm terminates after $K = \frac{1}{\delta}$ steps and, w.h.p.,
\begin{align}
    G(\mathbf{y}_K) \geq (1 - (1 - \delta)^{\frac{1}{\delta}})G(\vc{y}^*) \geq (1 - \frac{1}{e})G(\vc{y}^*)
\end{align}
where $\vc{y}^*$ is an optimal solution to (\ref{eq: multilinProb}).
\end{theorem}
\begin{algorithm}[!t]
    \caption{the Continuous Greedy algorithm}\label{alg:cont-greed}
    \begin{algorithmic}[1] 
     \State Input: {$G: P(\mathcal{M}) \rightarrow \reals_+$, $0 < \gamma \leq 1$}
      \State $\vc{y}_0\gets 0, \, t\gets 0, \, k\gets 0$
      
     \While{$t<1$}
     \State $\vc{m}_k \gets \argmax_{\vc{m} \in P(\mathcal{M})} \langle \vc{v}, \nabla G(\vc{y}_k) \rangle $
      \State   $\gamma_k \gets \min(\gamma, 1-t)$
     \State  $\vc{y}_{k+1} \gets \vc{y}_k + \gamma_k\vc{m}_k$, $t \gets t + \gamma_k$, $k \gets k + 1$
     \EndWhile
     \State
   \Return {$\vc{y}_k$}
    \end{algorithmic}
\end{algorithm}
\section{Multilinear Functions.}\label{sec:multilin}
In practice, estimating $G$ (and, through \eqref{eq: partial_derivative}, its gradient) via sampling poses a considerable computational burden. Attaining the guarantees of Thm.~\ref{thm: sampler_approx} requires the number of samples per estimate to grow as $N^2d^4$, that can quickly become prohibitive.

In some cases, however, the multilinear relaxation $G(\vc{y})$ has a polynomially-computable closed form. A prominent example is the coverage function, that arises in several different contexts \cite{ageev2004pipage, karimi2017stochastic}. Let $U=\{\mathcal{J}_1,\ldots,\mathcal{J}_n\}$ be a collection of subsets of some ground set $V=\{1,\ldots,N\}$. 
The coverage  $f: \{0, 1\}^{N} \rightarrow \reals_+$ is: 
\begin{equation}
    f(\vc{x}) = \textstyle\sum_{\ell=1}^n \left(1 - \prod_{i \in \mathcal{J}_{\ell}} (1 - x_i)\right)\label{eq:covform}.
\end{equation}
It is easy to confirm that:
\begin{align}
\nonumber    G(\vc{y}) &= \mathbb{E}_{\vc{x} \sim \vc{y}}[f(\vc{x})] = \mathbb{E}_{\vc{x} \sim \vc{y}}\big[\sum_{\ell=1}^n \big(1 - \prod_{i \in \mathcal{J}_{\ell}} (1 - x_i)\big)\big]\\
    &= \sum_{\ell=1}^n\big(1 - \prod_{i \in \mathcal{J}_{\ell}} (1 - \mathbb{E}_{\vc{x} \sim \vc{y}}[x_i])\big) 
  = f(\vc{y})\label{eq:eq}.
\end{align}
In other words, the multilinear relaxation evaluated over $\vc{y}\in [0,1]^N$ is actually equal to  $f(\vc{y})$, when the latter has form \eqref{eq:covform}. Therefore, computing it does not require sampling; crucially, \eqref{eq:covform} is $O(nN)$, i.e., polynomial in the input size. 

This clearly has a computational advantage when executing the continuous greedy algorithm. In fact, \eqref{eq:eq} generalizes to a broader class of functions: it holds as long as the objective $f$ is, itself, multilinear. Formally, a function, $f: \reals^N \rightarrow \reals$ is multilinear if it is affine w.r.t.~each of its coordinates \cite{broida1989comprehensive}. 
Put differently, multilinear functions are polynomial functions in which the degree of each variable in a monomial is at most $1$; that is, multilinear functions can be written as:
\begin{equation} \label{eq:multi}
    g(\vc{x}) = \textstyle\sum_{\ell \in \mathcal{I}} c_{\ell} \prod_{i \in \iset{J}{\ell}} x_i,
\end{equation}
where $c_{\ell}~\in~\reals$ for $\ell$ in some index set $\mathcal{I}$, and subsets $\iset{J}{\ell}~\subseteq~V$.\footnote{By convention, if $\mathcal{J}_{\ell} = \emptyset$, we set $\prod_{i\in\mathcal{J}_{\ell}} x_i = 1$.} Clearly, both the coverage function \eqref{eq:covform} and the multilinear relaxation \eqref{eq:multilinear} are multilinear in their respective arguments.

Eq. \eqref{eq:eq} generalizes to \emph{any multilinear function}. In particular:
\begin{lemma} \label{lem:relaxation_of_multi}
Let $f:\reals^N \rightarrow \reals_+$ be a multilinear function and let $\vc{x} \in \{0, 1\}^N$ be a random vector of independent Bernoulli coordinates parameterized by $\vc{y}\in~[0, 1]^N$. Then, 
$G(\vc{y}) = \mathbb{E}_{\vc{x}\sim\vc{y}}[f(\vc{x})] = f(\vc{y}).$ 
\end{lemma}
The proof can be found in \fullversion{\cite{ozcan2021submodular}}{App.~\ref{app:proof_relaxation_of_multi}}. Lem.~\ref{lem:relaxation_of_multi} immediately implies that all polytime-computable, submodular multilinear functions behave like the coverage function: computing their multilinear relaxation \emph{does not require sampling}. Hence, continuous greedy admits highly efficient implementations in this setting. Our main contribution is to extend this to a broader class of functions, by leveraging Taylor series approximations. We discuss this in detail in the next section.
\section{Main Results}\label{sec:mainres}
\begin{table}[t]
\caption{\footnotesize{Notation Summary}}
\resizebox{\linewidth}{!}{
    \begin{tabular}{l l} 
    \hline
    $\mathbb{R}$ & Set of real numbers\\
    $\mathbb{R}_+$ & Set of non-negative real numbers\\
    $G(V,E)$ & Graph $G$ with nodes $V$ and edges $E$\\
    $V$ & Ground set of $N$ elements\\
    $f$ & A monotone, submodular set function\\
    $\mathcal{I}$ & Collection of independent sets in $2^V$\\
    $\mathcal{M}$ & Matroid denoting the $(V,\mathcal{I})$ pair\\
    \texttt{conv$(\cdot)$} & Convex hull of a set\\
    $k$ & Cardinality constraint of a uniform matroid\\
    $\vc{x}$ & Global item placement vector of $x_{i}$'s in $\{0,1\}^{N}$\\
    $[\vc{x}]_{+i}$ & Vector $\vc{x}$ with the $i$th coordinate set to $1$\\
    $[\vc{x}]_{-i}$ & Vector $\vc{x}$ with the $i$th coordinate set to $0$\\
    $y_{i}$ & Probability of $i \in S$\\
    $\vc{y}$ & Vector of marginal probabilities $y_i$'s in $[0,1]^{N}$\\
    $G(\vc{y})$ & Multilinear extension with marginals $\vc{y}$\\
    $h_i$ & An analytic function\\
    $g_i$ & A multilinear function\\
    $w_i$ & Weights in $\reals$\\
    $\hat{h}_{L}$ & Polynomial estimator of $h_i$ of degree $L$\\
    $R_{i,L}$ & Residual error of the estimator $\hat{h}_{L}$\\
    $\hat{f}_{L}(\vc{x})$ & Polynomial estimator of $f(\vc{x})$ of degree $L$\\
    $R_L(\vc{x})$ & Residual error vector of the polynomial estimator $\hat{f}_{L}(\vc{x})$\\
    $\epsilon_{i, L}(\vc{y})$ & Residual error of the estimator $\partial \widehat{G(\vc{y})}/\partial y_i$\\
    $\varepsilon(L)$ & Bias of the estimator $\widehat{\nabla G(\vc{y})}$\\
    \hline
    & \textbf{Influence Maximization}\\
    \hline
    $M$ & Number of cascades\\
    \hline
    & \textbf{Facility Location} \\
    \hline
    $V$ & Number of facilities\\
    $M$ & Number of customers\\
    \hline
    & \textbf{Summarization}\\
    \hline
    $M$ & Number of partitions\\
    \hline
\end{tabular}}
\vspace*{-10pt}
\end{table}
In this section, we show that Eq.~\eqref{eq:eq} can be extended to submodular objectives that can be expressed 
via compositions of analytic functions and multilinear functions. In a nutshell, our approach is based on two observations: (a) when restricted to binary values, polynomials of multilinear functions are themselves multilinear functions, and (b) analytic functions are approximated at arbitrary accuracy via polynomials. Exploiting these two facts, we approximate the multilinear relaxation of an arbitrary analytic function via an appropriate Taylor series; the resulting approximation is multilinear and, hence, directly computable without sampling.
\subsection{Motivation and Intuition.} We begin by establishing that polynomials of multilinear functions are themselves multilinear functions, when restricted to binary values. Formally:
\begin{lemma} \label{lem:closed_multi}
The set of multilinear functions restricted over the domain $\{0, 1\}^N$ is closed under addition, multiplication, and multiplication with a scalar. 
\end{lemma}
 Put differently, multilinear functions restricted over the domain $\{0, 1\}^N$ form both a ring and a vector space. The proof of  Lem.~\ref{lem:closed_multi} can be found in \fullversion{\cite{ozcan2021submodular}}{App.~\ref{app:proof_closed_multi}}. It is important to note  that multilinear functions are closed under multiplication only when restricted to domain $\{0,1\}^N$. The general set of multilinear functions $f: [0, 1]^N \rightarrow \reals_+$ is \emph{not} closed under multiplication. 

Lem.~\ref{lem:closed_multi} has the following implication. Consider a submodular function $f:\{0,1\}\to\reals_+$ of the form
$f(\vc{x}) = h(g(\vc{x})) $
where $g:\reals^N\to\reals$ is a multilinear function, and $h:\reals\to\reals_+$ is an analytic function (e.g., $\log$, $\exp$, $\sin$, etc.). As $h$ is analytic, it can be approximated by a polynomial $\hat{h}$ around a certain value in its domain. 
This gives us a way to estimate the multilinear relaxation of $f$ without sampling. First, we  approximate $f$ by replacing $h$ with $\hat{h}$, getting
$\hat{f}=\hat{h}(g)$.  As  $\hat{f}$ is the polynomial of a multilinear function restricted to $\{0,1\}^N$, by Lem.~\ref{lem:closed_multi}, $\hat{f}$ \emph{can also be expressed as a multilinear function}. Thus, $G$ can be estimated \emph{without sampling} via the estimator $\hat{G}(\vc{y}) \triangleq \hat{f}(\vc{y})$.

In the remainder of this section, we elaborate further on construction, slightly generalizing the setup, and providing formal approximation guarantees.
\subsection{Assumptions.}
Formally, we consider set functions $f:\{0,1\}^N\to\reals_+$ that satisfy two assumptions:
\begin{assumption} \label{asmp: mon_sub}
Function $f:\{0,1\}^N\to \reals_+$ is monotone and submodular.
\end{assumption}
\begin{assumption} \label{asmp: f_isInForm}
Function $f: \{0,1\}^{N} \to \reals_+$ has  form \begin{align}\label{eq:masterform}
    f(\vc{x})=\textstyle\sum_{j=1}^{M} w_j h_j(g_j(\vc{x})),
    \end{align}
for some $M\in\naturals$, and $w_j \in \reals$,  $h_j: [0, 1] \rightarrow \reals_+$, 
and $g_j: [0,1]^N\rightarrow [0,1]$, for $j\in\{1,\ldots,M\}$.
Moreover, for every $j\in\{1,\ldots,M\}$, the following hold:
\begin{enumerate}
    \item Function $g_j:[0,1]^N\to[0,1]$ is  multilinear.
    \item 
    There exists a polynomial $\hat{h}_{L}:[0,1] \to \reals$ of degree $L$ for $L\in \naturals$, such that $|h_j(s)- \hat{h}_{L}(s)| \leq R_{j,L}(s)$, where
$\lim_{L\to \infty} R_{j,L}(s)=0, $ 
for all $s\in[0,1]$.
\end{enumerate}
\end{assumption}
Asm.~\ref{asmp: f_isInForm} implies that $f$ can be written as a linear combination of compositions of analytic functions $h_j$ with multilinear functions $g_j$. The former can be arbitrarily well approximated by polynomials of degree $L$; any residual error from this approximation converges to zero as the degree of the polynomial increases.

Tab.~\ref{table: problems} summarizes several problems that satisfy Assumptions~\ref{asmp: mon_sub} and~\ref{asmp: f_isInForm}. We review each of these problems in more detail in Sec.~\ref{sec:examples}; in the remainder of this section, we provide approximation guarantees for objectives that satisfy these two assumptions. 
\begin{table*}[ht] 
\caption{Summary of problems satisfying Assumptions~\ref{asmp: mon_sub} \& \ref{asmp: f_isInForm}.}
\centering \label{table: problems}
\resizebox{\textwidth}{!}{
    \begin{tabular}{ |c|c|c|c|c|c| } 
     \cline{2-6}
     \multicolumn{1}{c}{} 
     & \multicolumn{1}{|c|}{\thead{Input}} 
     & \multicolumn{1}{|c|}{\thead{$g_j: \{0, 1\}^{|V|} \rightarrow [0, 1]$\\
                                   $\vc{x} \rightarrow g_j(\vc{x})$}} 
     & \multicolumn{1}{|c|}{\thead{$h_j: [0, 1] \rightarrow \reals_+$\\
                                   $s \rightarrow h_j(s)$}} 
     & \multicolumn{1}{|c|}{\thead{$f:\{0, 1\}^{|V|} \rightarrow \reals_+$\\
                                   $\vc{x} \rightarrow f(\vc{x})$}} 
     & \multicolumn{1}{|c|}{\thead{Bias\\
                                   $\varepsilon(L)$}} \\ 
     \hline
     \thead{SM} 
     & \makecell{Partitions $\bigcup_{j=1}^M\{P_j\} = V$ \\
                 weights $\vc{r}\in\reals_+^{N}$, and $\sum_{i=1}^N r_i=1$
                 } 
     & $\sum\limits_{i \in P_j}  r_i x_{i}$ 
     & $\log( 1+s)$ 
     & $\sum\limits_{j=1}^{M} h(s_j)$ 
     & $\frac{M\sqrt{N}}{(L+1) 2^{L}}$ \\
     \hline
     \thead{IM} 
     & \makecell{Instances $G = (V, E)$\\
                 of a directed graph, \\
                 partitions $\{P_{v}^j\}_{j=1}^{N} \subset V$
                 } 
     & $\sum\limits_{i \in V}\frac{1}{N}\Big(1 - \prod\limits_{u \in P_{i}^j}(1-x_u)\Big)$ 
     & $\log( 1+s)$ 
     & $\frac{1}{M}\sum\limits_{j=1}^{M} h(s_j)$ 
     & $\frac{\sqrt{N}}{(L+1) 2^{L}}$ \\
     \hline
     \thead{FL} 
     & \makecell{Complete weighted bipartite\\
                 graph $G = (V \cup V')$\\
                 weights $w_{i_\ell, j} \in [0, 1]^{N \times M}$
                 } 
     & $\sum\limits_{\ell=1}^{N}(w_{i_\ell, j}-w_{i_{\ell+1}, j})\left(1-\prod\limits_{k=1}^\ell(1-x_{i_k})\right)$ 
     & $\log(1+s)$ 
     & $\frac{1}{M} \sum\limits_{j=1}^{M} h(s_j)$ 
     & $\frac{\sqrt{N}}{(L+1) 2^{L}}$ \\
     \hline
     \thead{CN} 
     & \makecell{Graph $G = (V, E)$,\\
                 service rates $\mu \in \reals_+^{M}$, \\
                 requests $r \in \mathcal{R}$, $P_j$ path of $r$, \\
                 arrival rates $\lambda \in \reals_+^{|\mathcal{R}|}$\\
                 } 
     & $\frac{1}{\mu_j}\sum_{r \in \mathcal{R}:j\in p^r} \lambda^r \prod_{k'=1}^{k_{p^r}(v)}(1-x_{p_k^r, i^r})$ 
     & $\frac{s}{1 - s}$ 
     & $\sum\limits_{j = 1}^M h(s_0) - \sum\limits_{j = 1}^M h(s_j)$ 
     & $2M\sqrt{{|V||\mathcal{C}|}}\frac{\bar{s}^{L+1}}{1-\bar{s}}$ \\
     \hline
    \end{tabular}}
    \vspace*{-10pt}
\end{table*}
\subsection{A Polynomial Estimator.}


Given a function $f$ that satisfies Asm.~\ref{asmp: f_isInForm}, we construct the  \emph{polynomial estimator of $f(\vc{x})$ of degree $L$} via  \begin{align}\hat{f}_{L}(\vc{x})\triangleq\textstyle \sum_{j=1}^M w_j \hat{h}_{L}(g_j(\vc{x})).\end{align}
By Lem.~\ref{lem:closed_multi}, function $\hat{f}_L:\{0,1\}^N\to \reals$ can be expressed as a multilinear function.
We define an estimator $\widehat{\nabla G_L}$ of the gradient  of the multilinear relaxation $G$ as follows: for all $i \in V$,
\begin{align}
      (\widehat{{\partial G_L}}/{\partial y_i})\big|_{\vc{y}} &= \mathbb{E}_{\vc{y}}[\hat{f}_{L}([\vc{x}]_{+i})] - \mathbb{E}_{\vc{y}}[\hat{f}_{L}([\vc{x}]_{-i})] \nonumber \\
    & \stackrel{\text{Lem.}~\ref{lem:relaxation_of_multi}}{=} \hat{f}_{L}([\vc{y}]_{+i}) - \hat{f}_{L}([\vc{y}]_{-i}). \label{eq: poly_estimator}
\end{align}
We  characterize the quality of this estimator via the following theorem, whose proof is in \fullversion{\cite{ozcan2021submodular}}{App.~\ref{proof: gradientBias}}:
\begin{theorem} \label{thm:gradientBias}
Assume that function $f$ satisfies Asm.~\ref{asmp: f_isInForm}. Let $\widehat{\nabla G_L}$ be the estimator of the multilinear relaxation given by \eqref{eq: poly_estimator}, and define
 $R_L(\vc{x}) \triangleq \sum_j |w_j|| R_{j,L}(g_j(\vc{x}))|$ for $\vc{x}\in\{0,1\}^N$.
Then, 
\begin{equation}\label{eq:estimator_bound}
    \big\|\nabla G(\vc{y}) - \widehat{\nabla G_L}(\vc{y})\big\|_2 \leq \|\epsilon_{L}(\vc{y})\|_2
\end{equation}
where $\epsilon_L(\vc{y}) = [\epsilon_{i, L}(\vc{y})]_{i=1 }^N\in\reals^N$ and
\begin{align}\label{eq:epsilonL}
    \epsilon_{i, L}(\vc{y}) \triangleq \mathbb{E}_{\vc{y}}[R_L([\vc{x}]_{+i})] + \mathbb{E}_{\vc{y}}[R_L([\vc{x}]_{-i})].
\end{align}
Moreover,
$\lim_{L \to \infty} \|\epsilon_L(\vc{y})\|_2= 0,$ uniformly on $ [0,1]^N$.
\end{theorem}
The theorem implies that, under Asm.~\ref{asmp: f_isInForm}, we can approximate $\nabla G$ arbitrarily well, uniformly over all $\vc{y}\in[0,1]^N$. This approximation can be used in continuous greedy, achieving the following guarantee:
\begin{theorem} \label{thm: main}
Assume a function $f: \{0,1\}^N~\rightarrow~\reals_+$ satisfies Assumptions~\ref{asmp: mon_sub}~and~\ref{asmp: f_isInForm}. 
Then, consider Alg.~\ref{alg:cont-greed}, in which $\nabla G(\mathbf{y}_K)$ is estimated via the polynomial estimator given in (\ref{eq: poly_estimator})
. Then,
\begin{align}
    G(\mathbf{y}_K) \geq \left(1-\frac{1}{e}\right)G(\mathbf{y}^*)-D\,\varepsilon(L)-\frac{P}{2K}, 
\end{align}
where $K=(1/\gamma)$ is the number of iterations, $\mathbf{y}^*$ is an optimal solution to (\ref{eq: multilinProb}), $D = \max_{\vc{y} \in P(\mathcal{M})}  \|\vc{y}\|_2$ is the diameter of the polymatroid, $\varepsilon(L) =\max_{k}\| \epsilon_L(\vc{y}_k) \|_2$ is the bias of the estimator, and $P = 2\max_{\vc{x} \in \mathcal{M}} f(\vc{x})$. 
\end{theorem}
The proof can be found in \fullversion{\cite{ozcan2021submodular}}{App.~\ref{proof: main}}. Uniform convergence in Thm.~\ref{thm:gradientBias} implies that the estimator bias $\varepsilon(L)$ converges to zero. Hence, Thm.~\ref{thm: main} implies that we can obtain an approximation arbitrarily close to $1-1/e$, by setting $L$ and $K$ appropriately. 

We note that Thm.~\ref{thm: main}  provides a tighter guarantee than the one achieved by Mahdian et al.~\cite{mahdian2020kelly} (see \fullversion{\cite{ozcan2021submodular}}{App.~\ref{app:compareBounds}} 
for a detailed comparison); in particular, they assume that  derivatives of functions $h_j$ are bounded; we make no such assumption. This is an important distinction, as none of the examples in Sec.~\ref{sec:examples}/Tab.~\ref{table: problems} have bounded derivatives (see \fullversion{\cite{ozcan2021submodular}}{App.~\ref{lem: R_iL_bound}}).
\subsection{Time Complexity.} For all examples in Tab.~\ref{table: problems}, the error $\varepsilon(L)$ decays exponentially with $L$. Hence, to achieve an approximation  $1-1/e+\varepsilon$, we must have $L=\Theta\left(\log\left(\frac{1}{\varepsilon}\right)\right)$. Hence, if multilinear functions $g_j$, $j\in \{1,\ldots,M\}$ are polynomially computable w.r.t $N$ (as is the case for our examples), the total number of terms in $\widehat{f}_L$ will be polynomial in both $N$ and $\frac{1}{\varepsilon}$. We further elaborate  on complexity issues in \protect\fullversion{\cite{ozcan2021submodular}}{App.~\ref{app:complexity}}.
\section{Examples.}\label{sec:examples}
In this section, we list three problems that can be tackled through our approach, also summarized in Tab.~\ref{table: problems}; we also review cache networks (CN) in \fullversion{\cite{ozcan2021submodular}}{App.~\ref{app:CN}}.
\subsection{Data Summarization (SM)\cite{lin2011class, mirzasoleiman2016fast}.}
In data summarization, ground set $V$ is a set of tokens, representing, e.g., sentences in a document or documents in a corpus. The goal is to select a ``summary'' $S\subseteq V$ that is representative of $V$. We present here the diversity reward function proposed by Lin and Bilmes\cite{lin2011class}. Assume that each token $i$ has a value $r_i\in [0,1]$, where $\sum_i r_i=1$. The summary $S$ should contain tokens of high value, but should simultaneously be diverse. The authors achieve this by partitioning $V$ to sets $\{P_j\}_{j=1}^M$, where each set $P_j\subset V$ contains tokens that are similar. They then seek a summary that maximizes 
\begin{equation}\label{eq:smob}
    f(S) = \textstyle\sum_{j=1}^M h\left(\sum_{i \in P_j \cap S} r_i\right),
\end{equation}
where $h:\reals_+\to\reals_+$ is a non-decreasing concave function (e.g., $h(s)=\log (1+s)$, $h(s)=s^\alpha$, where $\alpha<1$, etc.). Intuitively, the use of $h$ suppresses the selection of similar items (in the same $P_j$), even if they have high values, thereby promoting diversity.

Objective \eqref{eq:smob} is clearly of form \eqref{eq:masterform}. For example, for $h=\log(1+s)$, $f$ is 
 monotone and submodular  \cite{lin2011class}, and is the sum of  compositions of $h$ with multilinear functions $g_j(\vc{x})=\sum_{i\in P_j} r_i x_{i},$ as illustrated in Tab.~\ref{table: problems}.
 Moreover, $h$ is analytic and can be approximated within arbitrary accuracy by its $L^{\text{th}}$-order Taylor approximation around 1/2, given by:
\begin{equation} \label{eq: taylor_f_iL}
    \hat{h}_{L}(s) = \textstyle\sum_{\ell = 0}^L \frac{h^{(\ell)}(1/2)}{\ell!} (s - {1}/{2})^\ell.
\end{equation}
We show in \fullversion{\cite{ozcan2021submodular}}{App.~\ref{app: proof_bias_bound}} that this estimator ensures that $f$ indeed satisfies Asm.~\ref{asmp: f_isInForm}. Moreover, 
The estimator bias appearing in Thm.~\ref{thm: main} is also bounded:
\begin{theorem} \label{thm: epsilon_bound}
Assume a diversity reward function $f:~\{0,1\}^N \rightarrow \reals_+$ that is given by \eqref{eq:smob}, with $h(s)=\log(1+s)$. Then, consider the estimator $\widehat{\nabla G}(\mathbf{y}_K)$ given in (\ref{eq: poly_estimator}) using $\hat{h}_{L}(\vc{x})$,  the $L^{th}$ Taylor polynomial of $f(\vc{x})$ around $1/2$, given by \eqref{eq: taylor_f_iL}. Then, the bias of the estimator satisfies 
   $\varepsilon(L) \leq \frac{M\sqrt{N}}{(L+1) 2^{L}}.$
\end{theorem}
The proof of this theorem can be found in \fullversion{\cite{ozcan2021submodular}}{App.~\ref{app: proof_bias_bound}}. Our work directly allows for the optimization of such objectives over matroid constraints. For example, a partition matroid (distinct from $\{P_j\}_{j=1}^M$) could be used to enforce that no more than $k_\ell$ sentences come from $\ell$-th paragraph, etc. 
\subsection{Influence Maximization (IM) \cite{kempe2003maximizing, chen2009efficient}.} \label{sec: IM}
Influence maximization problems can be expressed as weighted coverage functions (see, e.g., \cite{karimi2017stochastic}). In short, given a directed graph $G = (V, E)$, we wish to maximize the expected fraction of nodes reached if we infect a set of nodes $S\subseteq V$ and the infection spreads via the Independent Cascade (IC) model \cite{kempe2003maximizing}.
In our notation this objective can be written as
\begin{align}f(\vc{x}) = \textstyle\frac{1}{M}\sum_{j=1}^{M} \frac{1}{N}\sum_{v \in V}\left(1 - \prod_{i \in P_v^j}(1-x_i)\right),\!\!\! \end{align}
where $P_v^j\subseteq V$ is the set of nodes reachable from $v$ in a random simulation of the IC model. 
This is a multilinear function. Our approach allows us to extend this  to maximizing the expectation of \emph{analytic functions} $h$ of the fraction of infected nodes. For example, for $h(s)=\log (1+s)$, we get:
\begin{equation} \label{eq: IM_gi}
    g_j(\vc{x}) =\textstyle \sum_{v \in V}\frac{1}{N}\big(1 - \prod_{i \in P_v^j}(1-x_i)\big), 
\end{equation}
for $j=1,\ldots,M$, and
\begin{equation} \label{eq: inf_max}
    f(\vc{x}) =\textstyle \frac{1}{M}\sum_{j=1}^{M} h\left(g_j(\vc{x}) \right).
\end{equation}
Functions $g_j:[0,1]^N\to [0,1]$ are multilinear, monotone submodular, and $O(N^2)$ computable, while $h:[0,1]\to\reals$  is non-decreasing and concave. 
As a result, (\ref{eq: inf_max}) satisfies Asm.~\ref{asmp: mon_sub}. 
Again, $h$ can be approximated within arbitrary accuracy by its $L^{\text{th}}$-order Taylor approximation around 1/2, given by \eqref{eq: taylor_f_iL}.
This again ensures that $f$ indeed satisfies Asm.~\ref{asmp: f_isInForm}. Moreover, 
we  bound the estimator bias appearing in Thm.~\ref{thm: main} as follows:
\begin{theorem} \label{thm: epsilon_bound_IM}
For function $f:~\{0,1\}^N \rightarrow \reals_+$ that  given by (\ref{eq: inf_max}), 
consider the estimator $\widehat{\nabla G}$   given in (\ref{eq: poly_estimator}) using  $\hat{h}_L$, the $L^{\text{th}}$-order Taylor approximation of $h$ around $1/2$, given by \eqref{eq: taylor_f_iL}. Then, the bias of estimator $\widehat{\nabla G}$  satisfies
 $    \varepsilon(L) \leq \frac{\sqrt{N}}{(L+1) 2^{L}}.$ 
\end{theorem}
The proof of the theorem can be found in \fullversion{\cite{ozcan2021submodular}}{App.~\ref{app: proof_bias_bound_IM}}. Partition matroid constraints could be used in this setting to bound the number of seeds from some group (e.g., males/females, people in a zip code, etc.).

\subsection{Facility Location (FL)\cite{mokhtari2018conditional, cornuejols1977location}.} 
Facility location  is another classic example of submodular maximization \cite{krause2014submodular}. Given a complete weighted bipartite graph $G = (V \cup V')$ and weights $w_{v, v'} \in [0, 1]$, $ v \in V$, $ v' \in V'$, we wish to maximize:
\begin{equation}\label{eq:originalFL}
    f(S) = \textstyle\frac{1}{M} \sum_{j=1}^{M} \max_{i \in S} w_{i, j}\, .
\end{equation}
Intuitively, $V$ and $V'$ represent facilities and customers respectively and $w_{v, v'}$ is the utility of facility $v$ for customer $v'$. The goal is to select a subset of facility locations $S\subset{V}$ to maximize the total utility, assuming every customer chooses the facility with the highest utility in the selection $S$. This too becomes a coverage problem 
by observing that $\max_{i \in S} w_{i, j}$ equals \cite{karimi2017stochastic}:
\begin{equation}
    g_j(\vc{x}) = \sum\limits_{\ell=1}^{N}(w_{i_\ell, j}-w_{i_{\ell+1}, j})\big(1-\prod\limits_{k=1}^\ell(1-x_{i_k})\big),\!\!\!\!
\end{equation}
where,
 for a given $j \in V'$, weights have been pre-sorted in a descending order as $w_{i_1,j} \geq \ldots \geq w_{i_n,j}$ 
and 
$w_{i_{n+1},j} \triangleq 0$. We can again extend this problem to maximizing analytic functions $h$ of the utility of a user. For example, for $h(s)=\log(1+s)$, we can maximize 
\begin{equation} \label{eq:FLlog}
    f(\mathbf{x}) = \textstyle\frac{1}{M} \sum_{j=1}^{M} \log\left(1+g_j(\vc{x}) \right).
\end{equation}
In a manner similar to the influence maximization problem, we can show that this function again satisfies Assumptions~\ref{asmp: mon_sub} and~\ref{asmp: f_isInForm}, using the $L^{\text{th}}$-order Taylor approximation of $g$, given by \eqref{eq: taylor_f_iL}. 
Moreover, as in Thm.~\ref{thm: epsilon_bound_IM}, the corresponding estimator bias is again $\varepsilon(L)\leq\frac{\sqrt{N}}{(L+1)2^L}$. We can again therefore optimize such an objective over arbitrary matroids, which can enforce, e.g., that no more than $k$ facilities are selected from a geographic area or some other partition of $V$.
\section{Experimental Study.} \label{sec:experiments}
\begin{table}[t]
\begin{center}
\resizebox{\linewidth}{!}{
    \begin{tabular}{|c|c|cccc|cc|c|}
    \hline
    \thead{instance} & \thead{dataset} & \thead{$M$} & \thead{$N$} & \thead{$\sum_{j=1}^M \mathcal{I}$} & \thead{$\bar{\mathcal{J}}$} & \thead{m} & \thead{k} & \thead{$f^*$}\\
    \hline
    \thead{IM} & \texttt{IMsynth1} & 1 & 200 & 200 & 5.2 & 10 & 3 & 0.3722\\
    \thead{IM} & \texttt{IMsynth2} & 1 & 200 & 200 & 5.1 &  10 & 3 & 0.6031\\
    \thead{FL} & \texttt{FLsynth1} & 200 & 200 & 40000 & 4.3 & 10 & 5 & 0.5197\\
    \thead{FL} & \texttt{MovieLens} & 100 & 100 & 10000 & 4.6 & 10 & 4 & 0.5430\\
    \thead{IM} & \texttt{Epinions} & 10 & 100 & 1000 & 3.2 & 2 & 2 & 0.5492\\
    \thead{SM} & \texttt{SMsynth1} & 5 & 200 & 200 & 7.4 & 2 & 10 & 0.7669\\
    \hline
    \end{tabular}}
\caption{{Datasets and Experiment Parameters.}}\label{tab:datasets}\end{center}\vspace*{-25pt}\end{table}
\subsection{Experiment Setup.}
\noindent We execute Alg.~\ref{alg:cont-greed} with sampling and polynomial estimators over $6$ different graph settings and $3$ different problem instances, summarized in Tab.~\ref{tab:datasets}. Our code is publicly available.\footnote{ \url{https://github.com/neu-spiral/WDNFFunctions}} 

\noindent\textbf{Influence Maximization.} We experiment on two synthetic datasets and one real dataset. For synthetic data, we generate two bipartite graphs with $|V_1|=|V_2|=100$, $|E|=400$ and $M=1$. Seeds are always selected from $V_1$. We select the edges across $V_1$ and $V_2$ u.a.r. (\verb|IMsynth1|) or by a power law distribution (\verb|IMsynth2|). We construct a partition matroid of $m=10$ equal-size partitions of $V_1$ and set $k=3$. The real dataset is the \texttt{Epinions} dataset \cite{epinions} on SNAP \cite{snapnets}. 
We use the subgraph induced by the top $N=100$ nodes with the largest out-degree and use the IC model \cite{kempe2003maximizing} with $M=10$ cascades. The probability for each node to influence its neighbors is set to $p=0.02$. We construct a matroid of $m=2$ equal-size partitions and set $k=5$. 



\begin{figure}[t]
\centering
\subfigure[\texttt{IMsynth1}]{
\begin{minipage}{0.46\linewidth}
\centering
\includegraphics[width=1\linewidth]{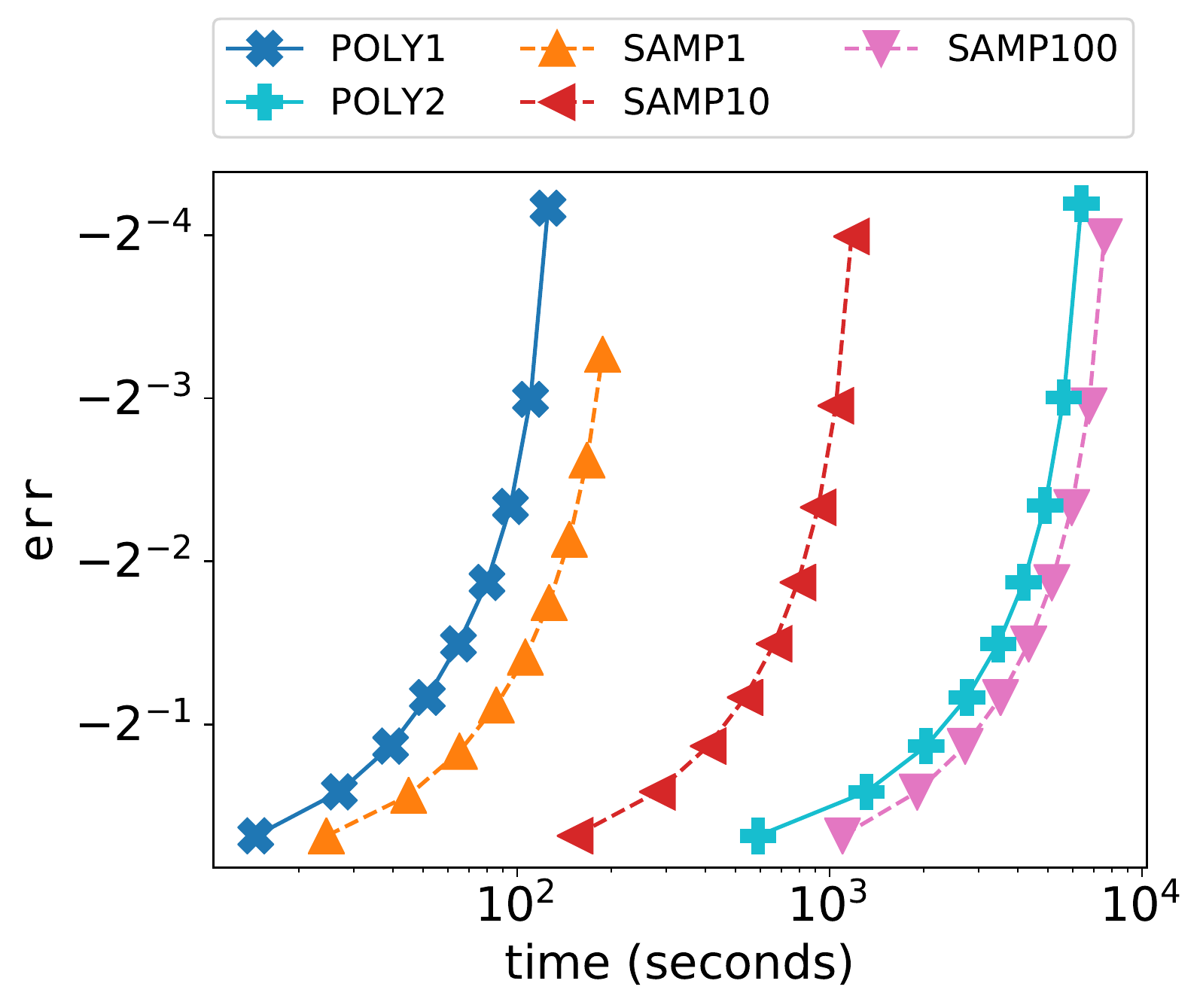}
\centering
\label{fig:IMsynth1_loglog}\vspace*{-10pt}
\end{minipage}
}
\subfigure[\texttt{IMsynth2}]{
\begin{minipage}{0.46\linewidth}
\centering
\includegraphics[width=1\linewidth]{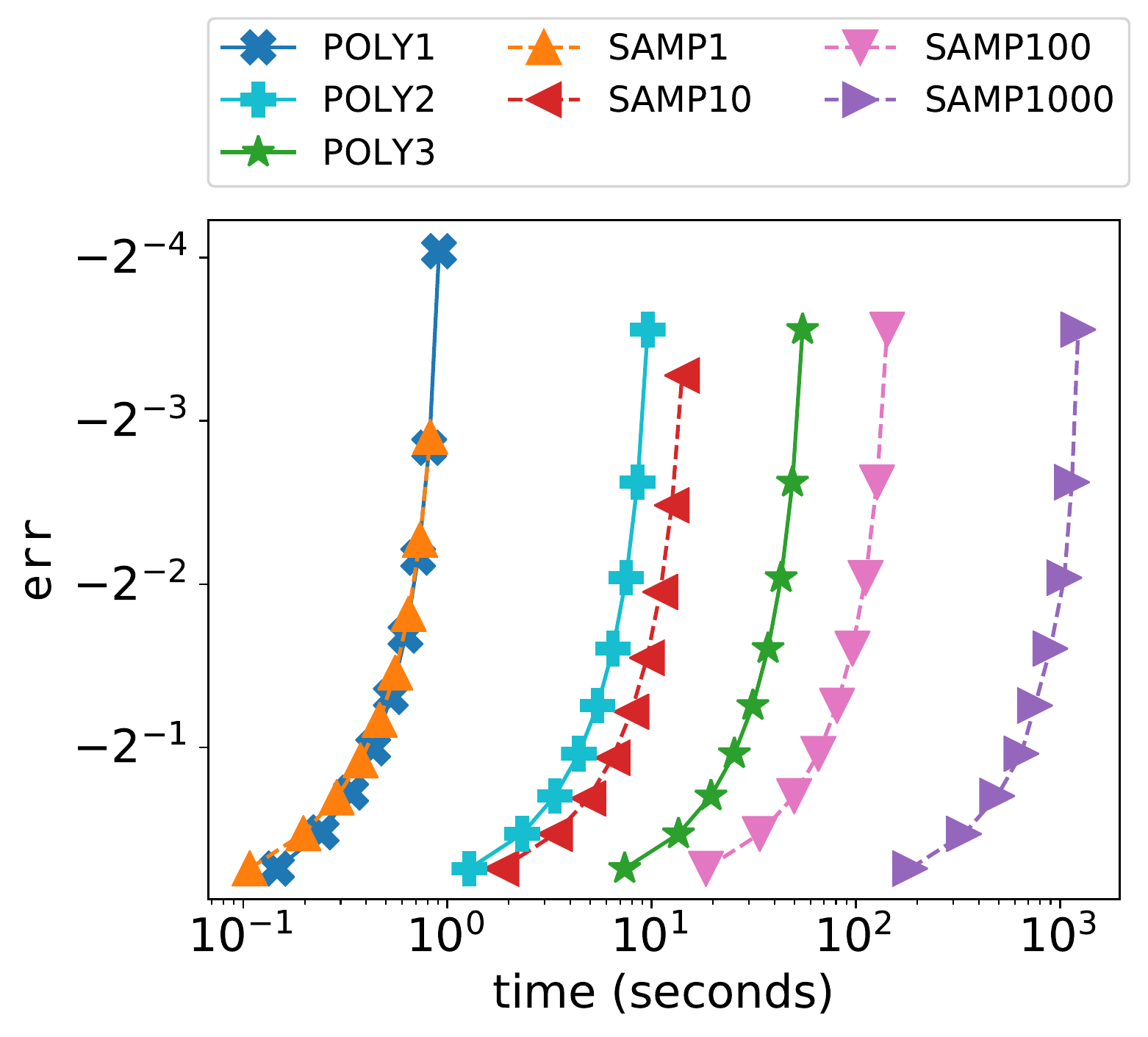}
\centering
\label{fig:IMsynth2_loglog}\vspace*{-10pt}
\end{minipage}
}
\subfigure[\texttt{FLsynth1}]{
\begin{minipage}{0.45\linewidth}
\centering
\includegraphics[width=1\linewidth]{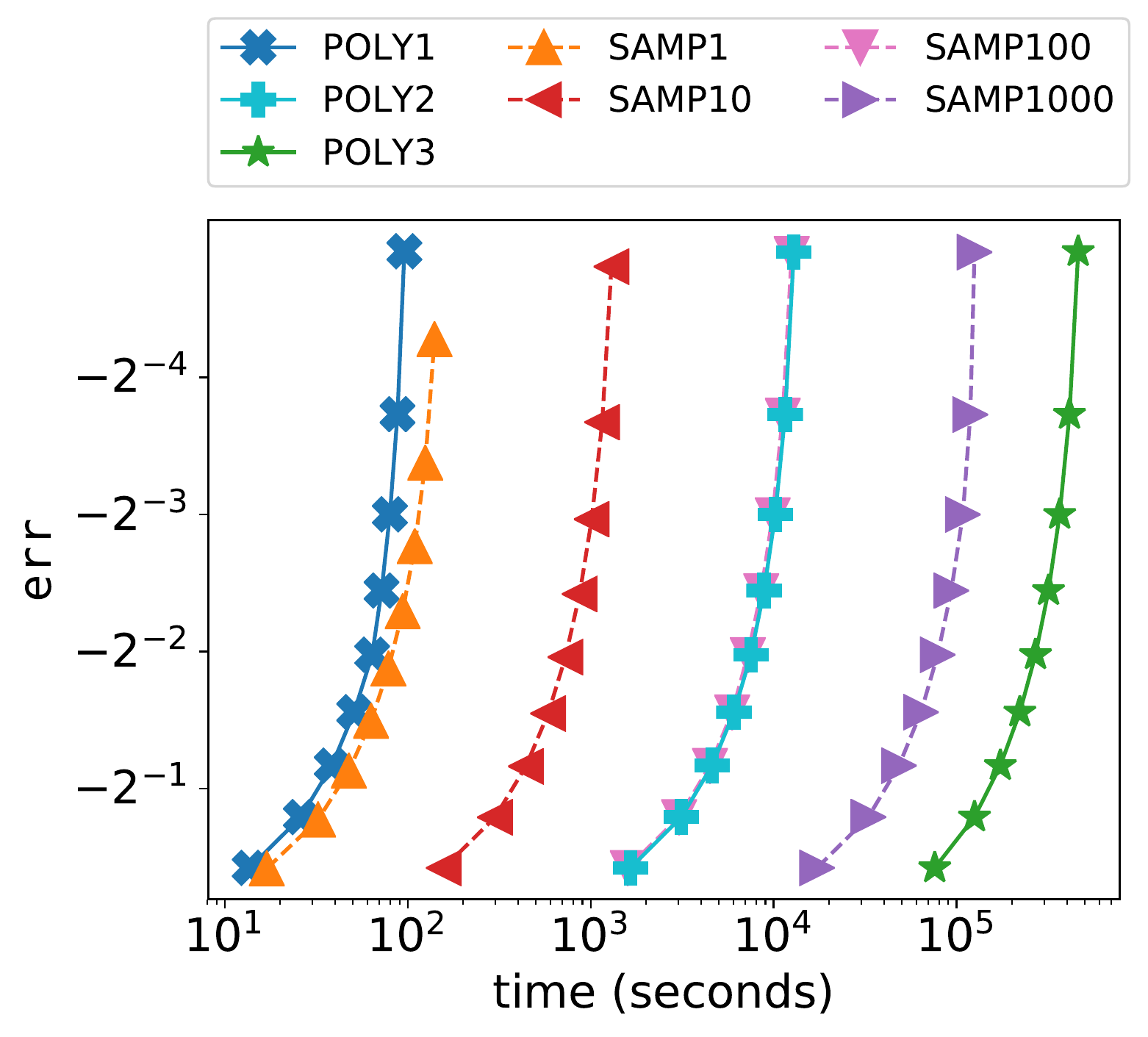}
\label{fig:FLsynth1_loglog}\vspace*{-10pt}
\end{minipage}
}
\subfigure[\texttt{MovieLens}]{
\begin{minipage}{0.45\linewidth}
\centering
\includegraphics[width=1\linewidth]{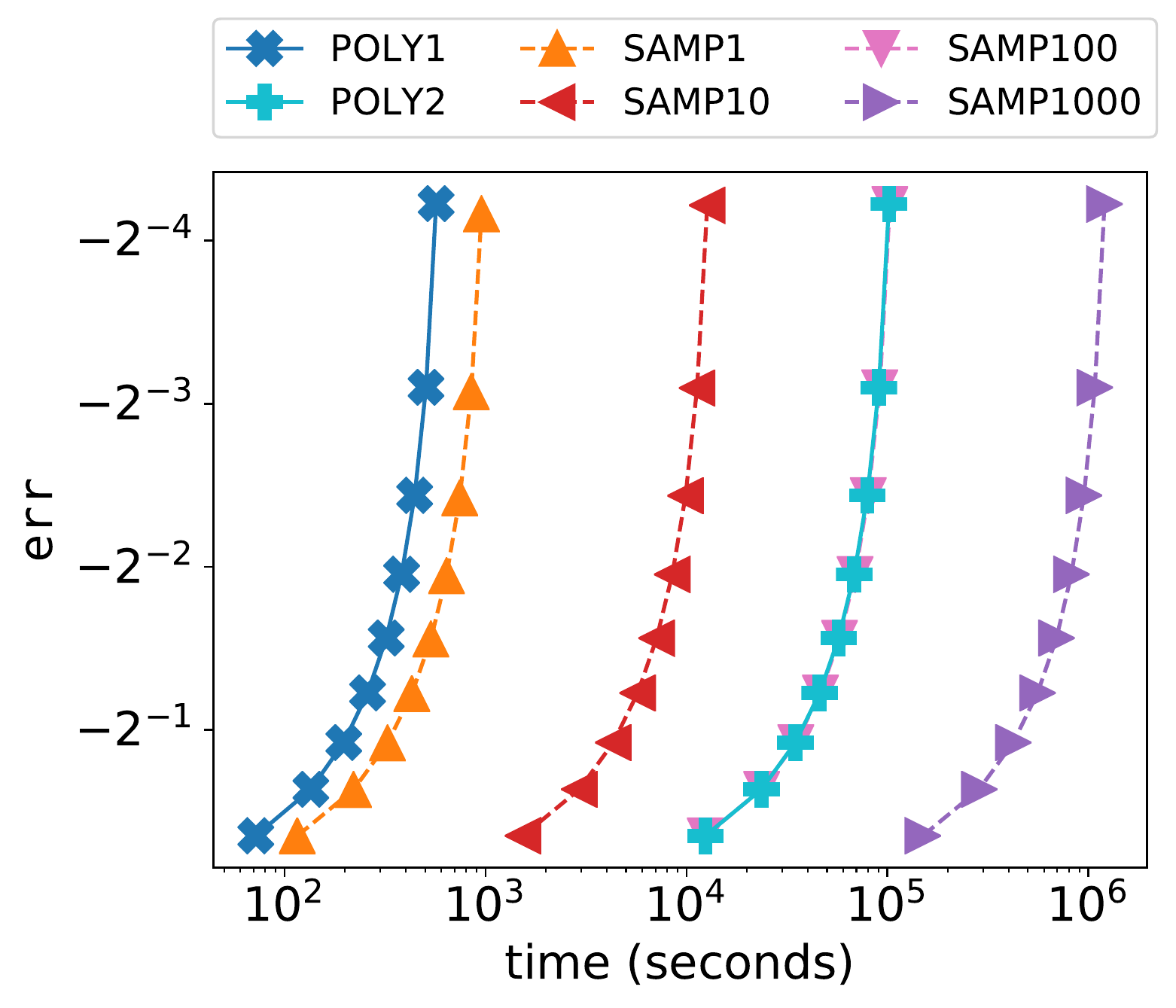}
\label{fig:MovieLens_loglog}\vspace*{-10pt}
\end{minipage}
}
\vspace*{-10pt}
\caption{Trajectory of the FW algorithm. Utility of the function at the current $\vc{y}$ as a function of time is marked for every $10$th iteration.} 
\label{fig:CGiters}
\end{figure}

\noindent\textbf{Facility Location.} We experiment on one synthetic and one real dataset. We generate a bipartite graph with $N=M=200$, $|E|=800$ and select the edges across $V$ and $V'$ u.a.r (\texttt{FLsynth1}). Weights of the edges ($w_{i, j}$) are selected randomly from $\{0.0, 0.2, 0.4, 0.6, 0.8, 1.0\}$. We construct a matroid of $m=10$ equal-size partitions and set to $k=4$. The real one is a subgraph of the \texttt{MovieLens} 1M dataset with the top $N=100$ users who rated the most movies and the $M=100$ movies chosen u.a.r. among the movies rated by the user who rated the most movies \cite{movielens}. In this problem, we treat movies as facilities, users as customers, and ratings as $w_{i, j}$.  
We construct a matroid of $m=10$ partitions by dividing movies according to their genres. We consider the first genre name listed if a movie belongs to multiple genres and we set $k=2$.

\noindent\textbf{Summarization.} We generate a synthetic dataset with $N=200$ nodes (\texttt{SMsynth1}). We assign a reward $r_i$ to each node $i$ u.a.r between $[0, 1]$ and divide each $r_i$ with $\sum_i r_i$. We divide the nodes into $M=5$ equal-size $P_j$. We construct a matroid of $m=2$ equal-size partitions and set $k=10$.   



\begin{figure}[t]
\centering
\subfigure[\texttt{IMsynth1}]{
\begin{minipage}{0.45\linewidth}
\centering
\includegraphics[width=1\linewidth]{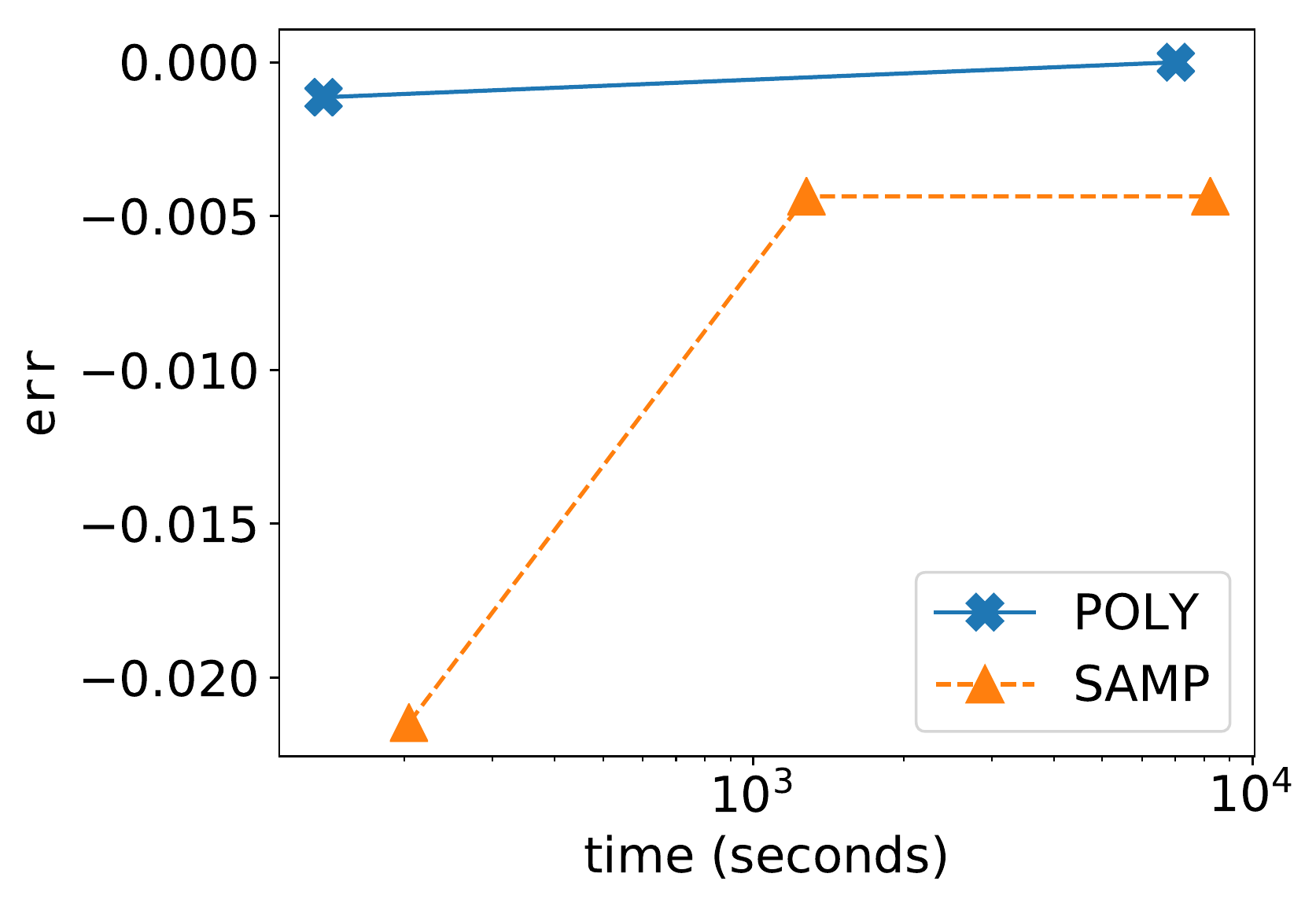}
\label{fig:IMsynth1_paretolog}\vspace*{-12pt}
\end{minipage}
}
\subfigure[\texttt{IMsynth2}]{
\begin{minipage}{0.45\linewidth}
\centering
\includegraphics[width=1\linewidth]{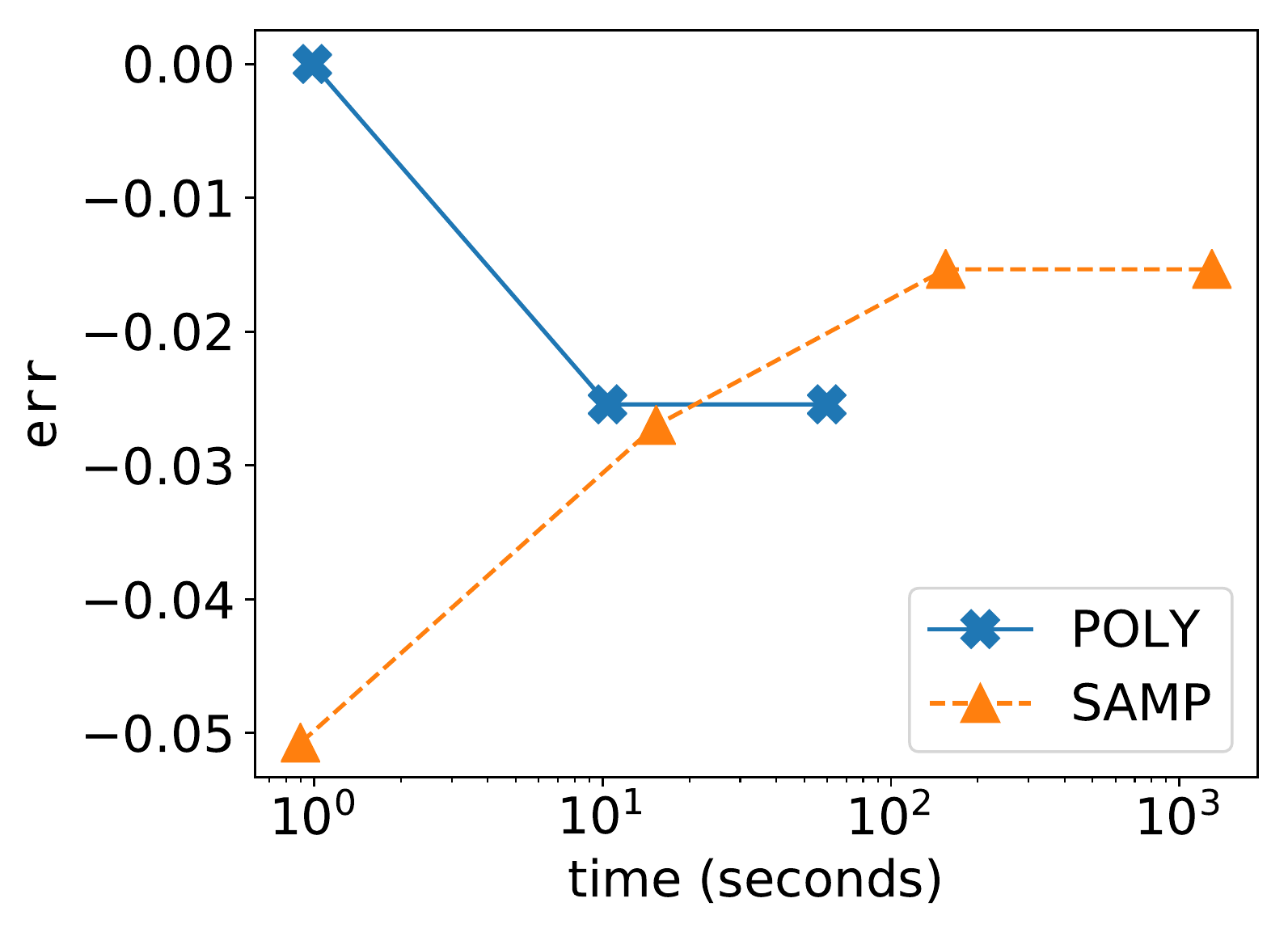}
\label{fig:IMsynth2_paretolog}\vspace*{-12pt}
\end{minipage}
}
\subfigure[\texttt{FLsynth1}]{
\begin{minipage}{0.45\linewidth}
\centering
\includegraphics[width=1\linewidth]{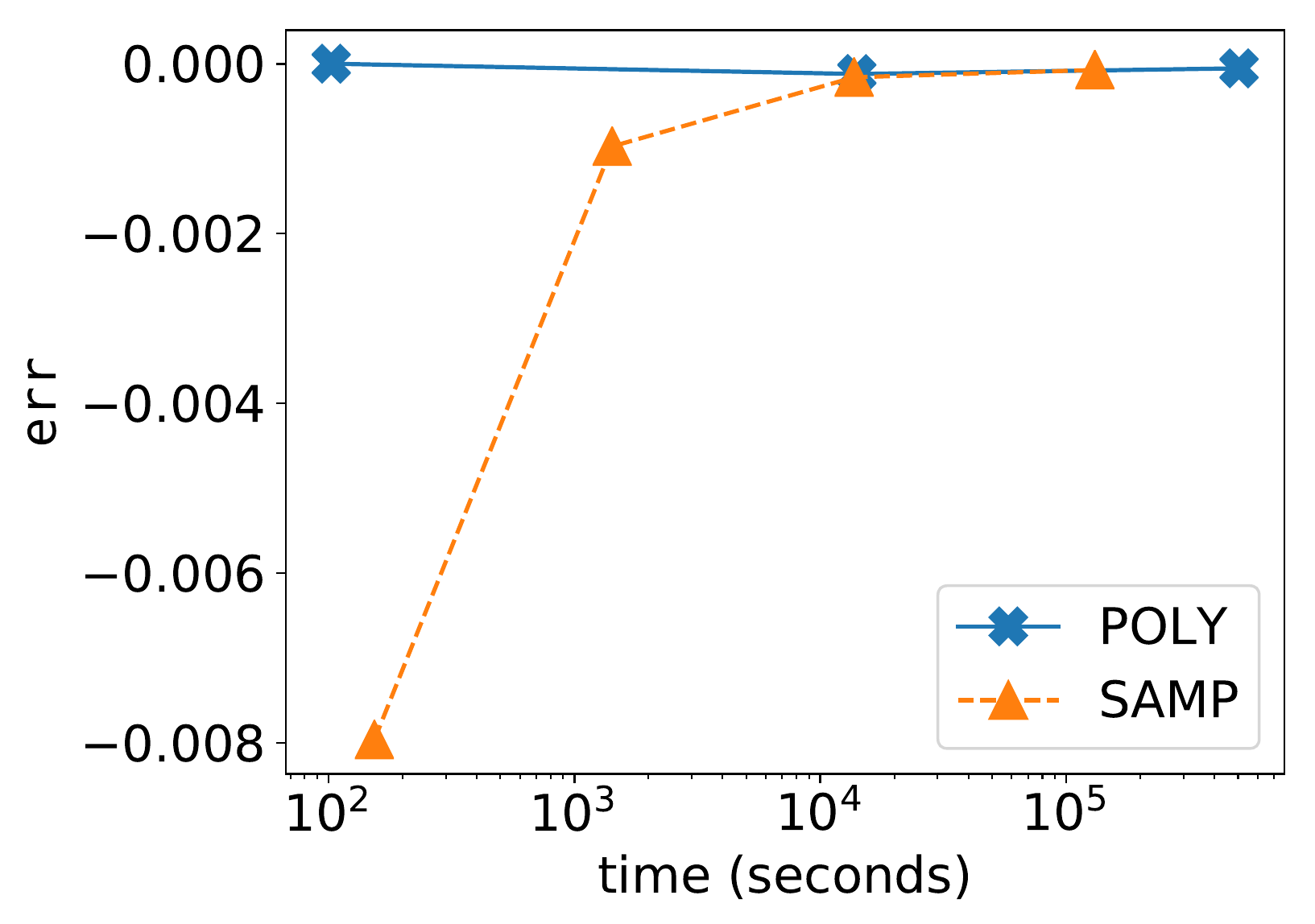}
\label{fig:FLsynth1_paretolog}\vspace*{-12pt}
\end{minipage}
}
\subfigure[\texttt{MovieLens}]{
\begin{minipage}{0.45\linewidth}
\centering
\includegraphics[width=1\linewidth]{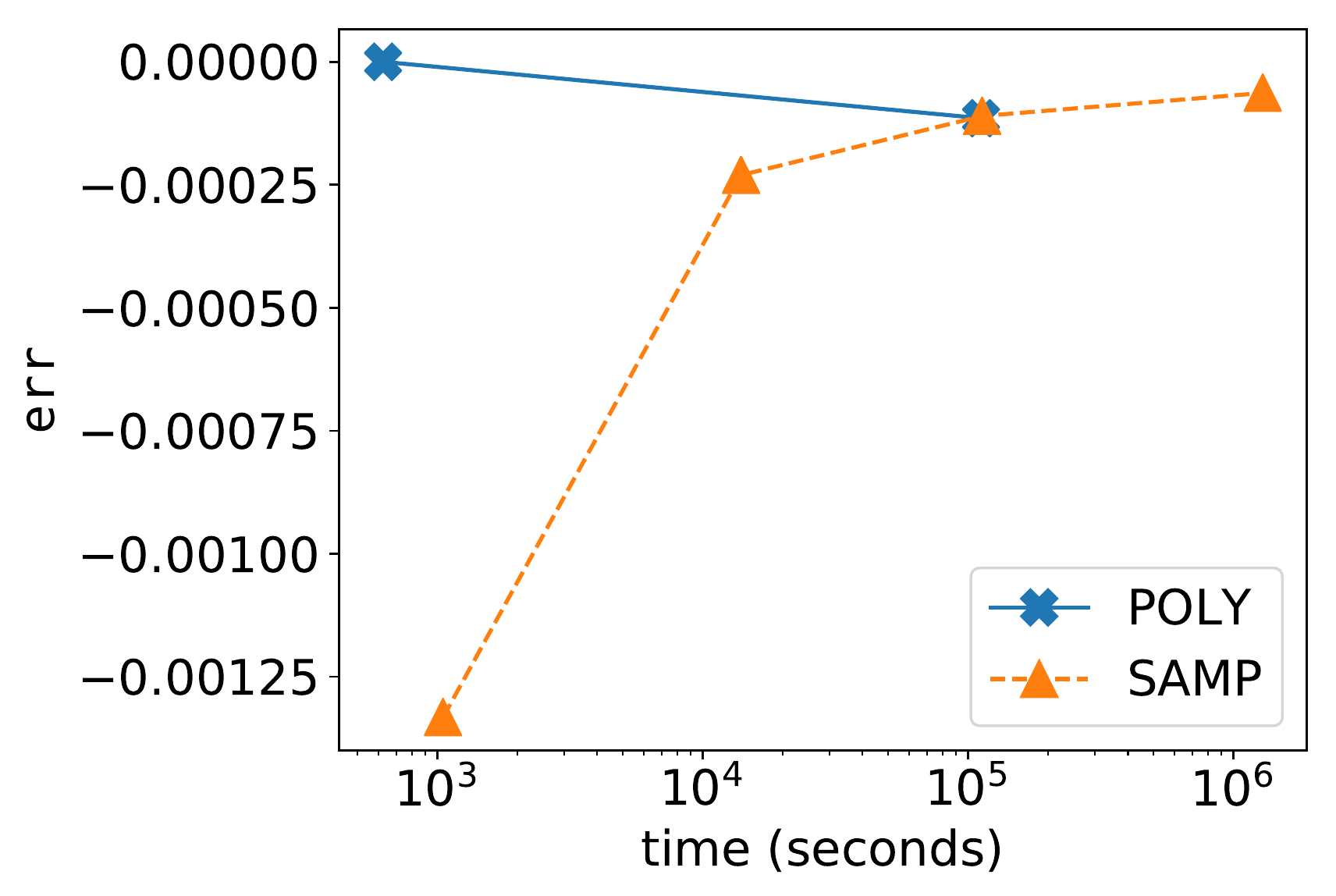}
\label{fig:MovieLens_paretolog}\vspace*{-12pt}
\end{minipage}
}
\subfigure[\texttt{Epinions}]{
\begin{minipage}{0.45\linewidth}
\centering
\includegraphics[width=1\linewidth]{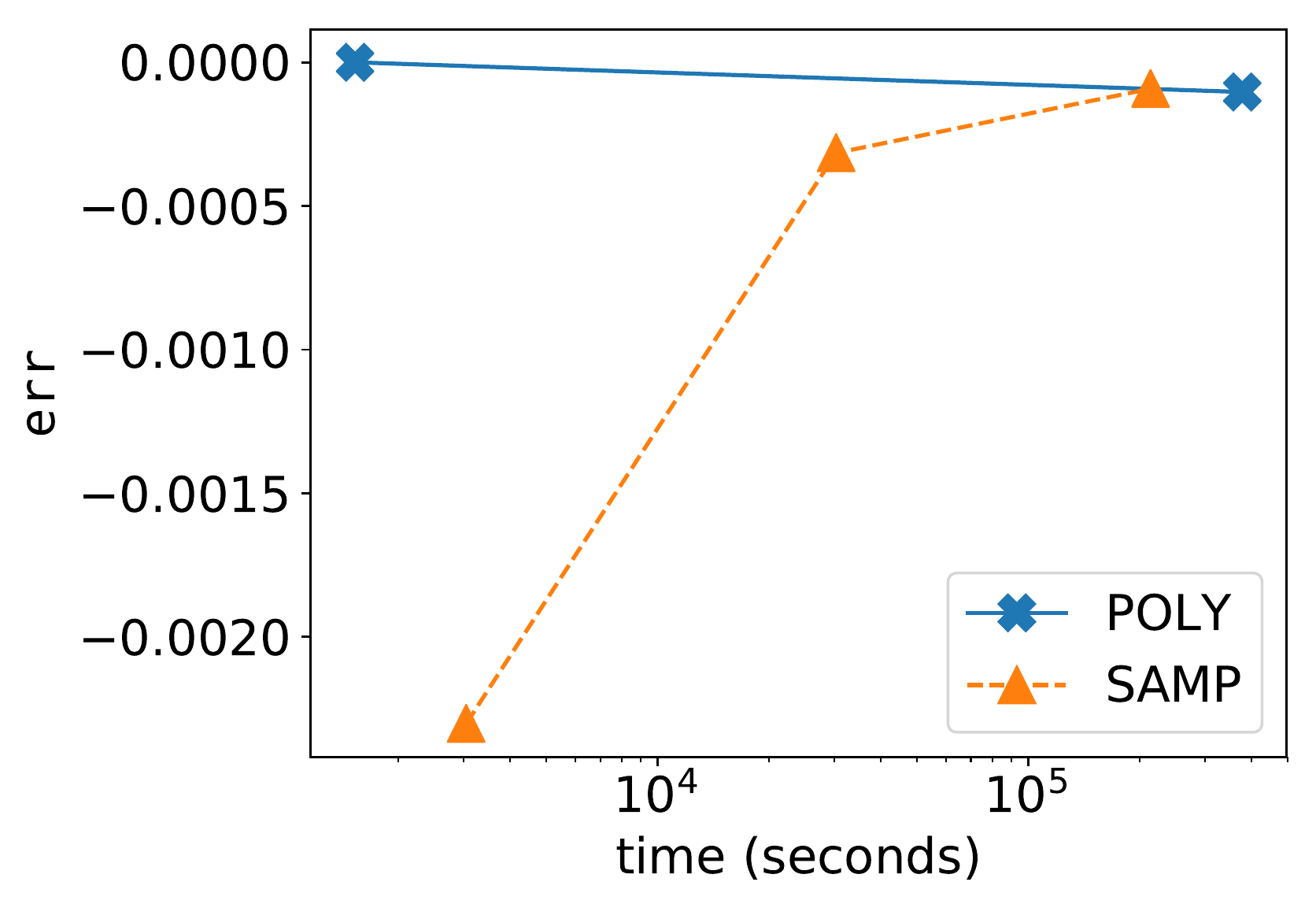}
\label{fig:Epinions_paretolog}\vspace*{-12pt}
\end{minipage}
}
\subfigure[\texttt{SMsynth1}]{
\begin{minipage}{0.45\linewidth}
\centering
\includegraphics[width=1\linewidth]{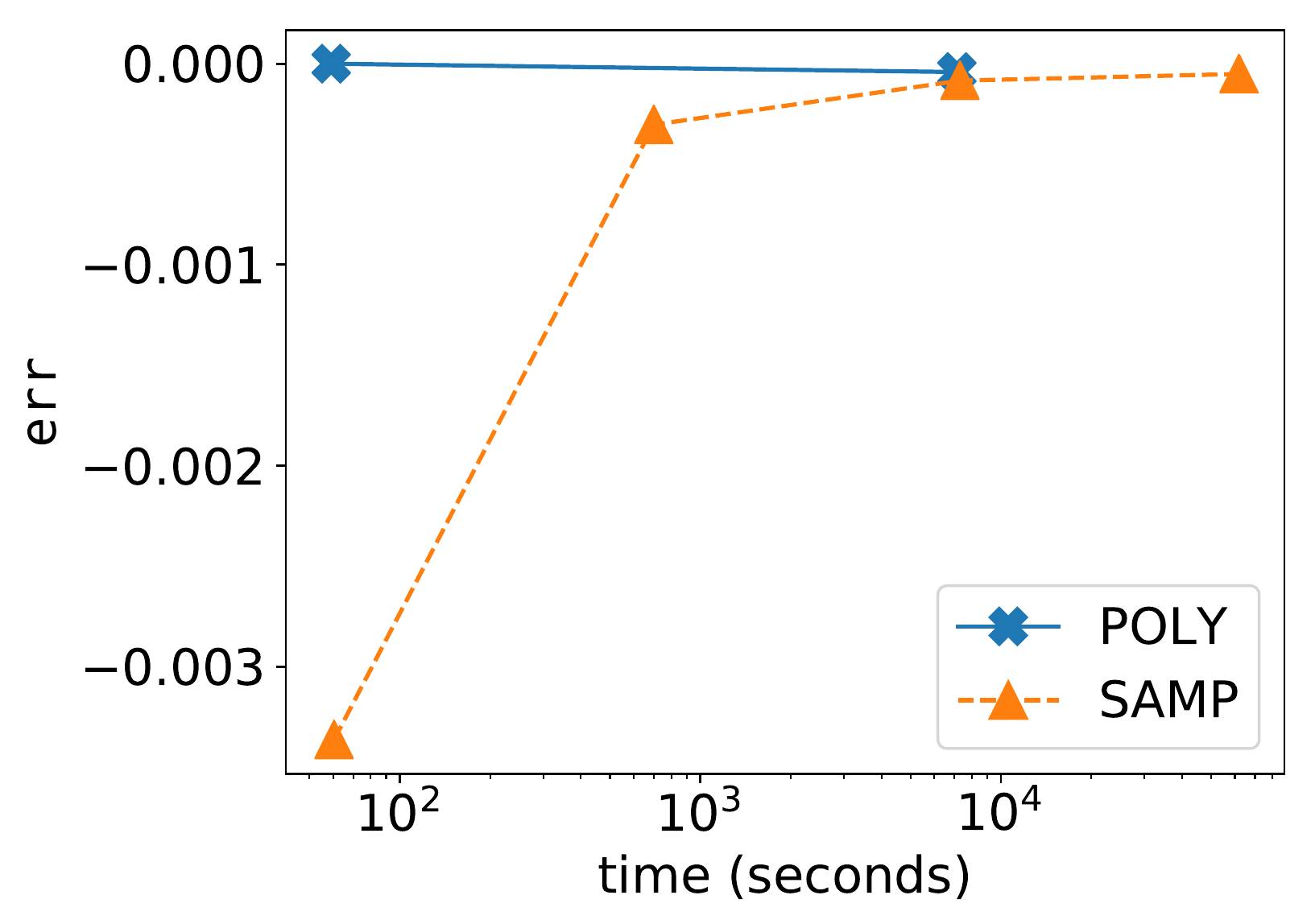}
\label{fig:SMsynth1_paretolog}\vspace*{-12pt}
\end{minipage}
}
\vspace*{-10pt}
\caption{Comparison of different estimators on different problems. Blue lines represent the performance of the POLY estimators and the marked points correspond to POLY1, POLY2, POLY3 respectively. Orange lines represent the performance of the SAMP estimators and the marked points correspond to SAMP1, SAMP10, SAMP100, SAMP1000 respectively.}
\label{fig:final_estimates}
\end{figure}
\noindent\textbf{Algorithms.}
We compare the performance of different estimators. These estimators are: (a) sampling estimator (SAMP) with $T~=~1, 10, 100, 1000$ and (b) polynomial estimator (POLY) with $L~=~1, 2, 3$. 

\noindent\textbf{Metrics.}
We measure the performance of the estimators via $\mathtt{err} = (f(\vc{y}) - f^*)/f^*$, where $f^* = \max f(\mathbf{y})$ is the maximum utility achieved using the best estimator for a given setting, and execution time. $f^*$ values are reported on Table \ref{tab:datasets}.

\subsection{Results.}
The trajectory of the normalized difference between the utility obtained at each iteration of the continuous greedy algorithm ($\mathtt{err}$) is shown as a function of time in Figure~\ref{fig:CGiters}. In Fig.~\ref{fig:IMsynth1_loglog}, we see that both POLY1 and POLY2 outperforms sampling estimators. Moreover, POLY1 is almost $60$ times faster than SAMP100. In Fig.~\ref{fig:IMsynth2_loglog}, POLY1 runs as fast as SAMP1 and outperforms all estimators. It is important to note that POLY3 runs $2.5$ times faster than SAMP1000. In Fig.~\ref{fig:FLsynth1_loglog}, POLY1 visibly outperforms SAMP1 and in Fig.~\ref{fig:MovieLens_loglog} polynomial estimators give comparable results to sampler estimators. Note that, even though small number of samples give comparable results, setting $T \leq 100$, is below the value needed to attain the theoretical guarantees of the continuous-greedy algorithm. These comparable results can be explained by the $1/2$ approximation guarantee of the greedy algorithm.

The $\mathtt{err}$ of the final results of the estimators are reported as a function of time in Figure~\ref{fig:final_estimates}. In all figures except Fig.~\ref{fig:IMsynth1_paretolog}, POLY1 outperforms other estimators in terms of time and/or utility whereas in Fig.~\ref{fig:IMsynth1_paretolog} POLY2 is the best performer. As the number of samples increases, the quality of the sampling estimators increases and they catch up with the polynomial estimators. However, considering the running time, POLY1 still remains the better choice.
\section{Conclusion.} \label{sec:conclusion}
We have shown that polynomial estimators can replace sampling of the multilinear relaxation. Our approach applies to other tasks, including rounding (see \fullversion{\cite{ozcan2021submodular}}{App.~\ref{app:pipage_poly}}) and stochastic optimization methods \cite{karimi2017stochastic}. For example, sampling terms of the polynomial approximation can extend our method to even larger problems. 
\bibliography{references}%
\bibliographystyle{ieeetr}
\newpage
\appendix
\section{Rounding}\label{sec:rounding}
Several poly-time algorithms can be used to round the fractional solution that is produced by Alg.~\ref{alg:cont-greed} to an integral $\vc{x}\in \domain$. We briefly review two such rounding algorithms: pipage rounding \cite{ageev2004pipage}, which is deterministic, and swap-rounding~\cite{chekuri2010dependent}, which is randomized. As in all the stated examples, the constraints are partition matroids (see Sec.~\ref{sec:submat}), here we limit our explanation to this case. For a more rigorous treatment, we refer the reader to \cite{ageev2004pipage}
for pipage rounding, and \cite{chekuri2010dependent} for swap rounding. 

 \noindent\textbf{Pipage Rounding.} This technique uses the following property of the multilinear relaxation $G$: given a fractional solution $\vc{y} \in P(\mathcal{M})$, there are at least two fractional variables $y_{i}$ and $y_{i'}$, where $i,i'\in B_j$ for some $j\in\{1,\ldots,m\}$, such that transferring mass from one to the other, $(1)$ makes at least one of them 0 or 1, $(2)$ the new $\hat{\mathbf{y}}$ remains feasible in $P(\mathcal{M})$, and $(3)$ $G(\hat{\mathbf{y}}) \geq G(\mathbf{y}(1))$, that is, the expected caching gain at $\hat{\mathbf{y}}$ is at least as good as $\mathbf{y}$. This process is repeated until $\hat{\mathbf{y}}$ does not have any fractional elements, at which point pipage rounding terminates and return $\hat{\mathbf{y}}$.
	This procedure has a run-time of $O(N)$, and since (a) the starting solution $\mathbf{y}$ is such that $$G(\mathbf{y
})\geq (1-{1}/{e})G(\mathbf{y}^*),$$ where $\mathbf{y}^*$ is an optimizer of $G$    in $P(\mathcal{M})$, and (b)
	each rounding step can only increase $G$, it follows that the final integral  $\hat{\mathbf{y}}\in \domain$ must satisfy
	\begin{equation*}
	f(\hat{\mathbf{y}})=G(\hat{\mathbf{y}})\geq G(\mathbf{y}) \geq (1-\frac{1}{e})G(\mathbf{y}^*) \geq (1-\frac{1}{e}) f(\mathbf{x}^*),
	\end{equation*}
       where $\mathbf{x}^*$ is an optimal solution to \eqref{eq:subMAX}. Here, the first equality holds because $f$ and $G$ are equal at  integral points, while the last inequality holds because \eqref{eq: multilinProb} is a relaxation of \eqref{eq:subMAX}, maximizing the same objective over a larger domain. 
       
       Note that pipage rounding requires evaluating the multilinear relaxation $G$. This can be done via a sampling estimator, but also using the Taylor estimator we have constructed in our work. We present approximation guarantees for pipage rounding using our estimator in App.~\ref{app:pipage_poly}.  	

 \noindent\textbf{Swap rounding.} In this method, given a fractional solution $\mathbf{y} \in P(\mathcal{M})$ produced by Alg.~\ref{alg:cont-greed} observe that it can be written as a convex combination of integral vectors in $\domain$, i.e., $\mathbf{y} = \sum_{k=1}^K \gamma_k \mathbf{m}_k,$ where $\gamma_k\in [0,1], \sum_{k=1}^K\gamma_k=1,$ and $\mathbf{m}_k \in\domain$. Moreover, by construction, each such vector $\mathbf{m}_k$ is maximal, i.e., all constraints in \eqref{eq:part_mat} are satisfied with equality.  

Swap rounding iteratively merges these constituent integral vectors, producing an integral solution. At each iteration $i$, the present integral vector $\mathbf{c}_k$ is merged with $\mathbf{m}_{k+1}\in \mathcal{M}$ into a new integral solution $\mathbf{c}_{k+1}\in \mathcal{M}$ as follows: if the two solutions $\mathbf{c}_k$, $\mathbf{m}_{k+1}$ differ at two indices  $i, i' \in B_j$, for some $j\in[m]$, (the former vector is 1 at element $i$ and 0 at $i'$, while the latter is 1 at $i'$ and 0 at $i$) the masses in the corresponding elements are swapped to reduce the set difference.  Either the mass (of 1) in the $i$-th element of  $\mathbf{c}_k$  is transferred to the $i$-th element of $\mathbf{m}_{k+1}$ and its $i'$ is set to 0, or the mass in the $i'$ element of  $\mathbf{m}_{k+1}$ is transferred to the  $i$-th element  in $\mathbf{c}_k$ and its $i$-th element is set to 0; the former occurs with probability proportional to   $\sum_{\ell=1}^{k}\gamma_{\ell}$, and the latter with probability proportional to $\gamma_{k+1}$. The swapping is repeated until the two integer solutions become identical; this merged solution becomes $\mathbf{c}_{k+1}$. This process terminates after $K-1$ steps, after which all the points $\mathbf{m}_k$ are merged into a single integral vector $\mathbf{c}_K\in\mathcal{M}$. 

Observe that, in contrast to pipage rounding, swap rounding does not require any evaluation of the objective $G$ during rounding. This makes swap rounding significantly faster to implement; this comes at the expense of the approximation ratio, however, as the resulting guarantee $1-1/e$ is in expectation.

\section{Proofs of Multilinear Function Properties}
\subsection{Proof of Lemma \ref{lem:relaxation_of_multi}} \label{app:proof_relaxation_of_multi}
As $g$ is multilinear, it can be written as $g(\vc{x}) = \sum_{\ell \in \mathcal{I}} c_{\ell} \prod_{i\in\iset{J}{\ell}}x_i$, for some subset $\mathcal{I}$, $c_{\ell} \in \reals_+$, and index sets $\iset{J}{\ell}~\subseteq~\{1, \ldots, n\}$. Then, \begin{align*}\mathbb{E}_{\vc{y}}[g(\vc{x})]&=\textstyle\sum_{\ell\in\mathcal{I}}c_{\ell}\mathbb{E}_{\vc{y}}\left[\prod_{i\in\iset{J}{\ell}}x_i\right]\\&=\textstyle\sum_{\ell\in\mathcal{I}}c_{\ell}\prod_{i\in\iset{J}{\ell}}\mathbb{E}_{\vc{y}}\left[x_i\right]=\sum_{\ell\in\mathcal{I}}c_{\ell}\prod_{i\in\iset{J}{\ell}}y_i\\&=g(\vc{y}).\end{align*}\hspace{\stretch{1}} \vspace*{-15pt}\qed
\subsection{Proof of Lemma \ref{lem:closed_multi}} \label{app:proof_closed_multi}
It is straightforward to see that the lemma holds for addition and multiplication with a scalar.

To proof that lemma holds for multiplication, let two multilinear functions $g_1, g_2: \{0, 1\}^N \rightarrow \reals_+$, given by
$
	g_1(\vc{x})=\textstyle\sum_{\ell\in\mathcal{I}_1}c_{\ell}\prod_{i\in\iset{J}{\ell}}x_i$ and
    $g_2(\vc{x})=\textstyle\sum_{\ell'\in\mathcal{I}_2}c_{\ell'}\prod_{i\in\iset{J}{\ell'}}x_i.
$
Observe that their product $g_1 \cdot g_2$ is
\begin{equation*}
    g_1(\vc{x})g_2(\vc{x})=\sum_{(\ell,\ell')\in\mathcal{I}_1\times\mathcal{I}_2}c_{\ell}c_{\ell'}\prod_{i\in\iset{J}{\ell}\cap\iset{J}{\ell'}}x_i^2\prod_{i\in\iset{J}{\ell}\triangle\iset{J}{\ell'}}x_i
\end{equation*}
where $\triangle$ is the symmetric set difference. Since $x_i~\in~\{0, 1\}$, $x_i^2=x_i$. Therefore,
\begin{equation*}
    g_1(\vc{x})g_2(\vc{x})=\textstyle\sum_{(\ell,\ell')\in\mathcal{I}_1\times\mathcal{I}_2}c_{\ell}c_{\ell'}\prod_{i\in\iset{J}{\ell}\cup\iset{J}{\ell'}}x_i
\end{equation*}
is multilinear.\hspace{\stretch{1}} \qed





\section{Proof of Theorem \ref{thm:gradientBias}} \label{proof: gradientBias}
We start by showing that the norm of the residual error vector of the estimator converges to $0$. 
Recall that, by Asm.~\ref{asmp: f_isInForm} the residual error of the polynomial estimation $\hat{h}_{j,L}(s)$ is bounded by $R_{j,L}(s)$. 
Thus, for functions $f:\{0,1\}^N\rightarrow\reals_+$ satisfying Asm.~\ref{asmp: f_isInForm}, 
we have that 
\begin{align}\label{eq:R_L}
    \lvert f(\vc{x})-\hat{f_L}(\vc{x})\rvert\leq R_L(\vc{x}),
\end{align}
where $R_L(\vc{x})\triangleq\sum_j |w_j||R_{j,L}(g_j(\vc{x}))|$. Since $\lim_{L\to\infty}R_{j,L}(s)=0$ for all $j\in M$ and $s\in[0,1]$, and $g_j(\vc{x})\in [0,1]$ for all $j$ and $\vc{x}$, we get that, for all $\vc{x}\in\{0,1\}^N$, \begin{align}\lim_{L\to\infty}\lvert f(\vc{x})-\hat{f_L}(\vc{x})\rvert\leq\lim_{L\to\infty}R_L(\vc{x})=0.\label{eq:rl}\end{align}
In fact, this convergence happens uniformly over all $\vc{x}\in \{0,1\}^N$, as  $\{0,1\}^N$ is a finite set. Moreover, 
\begin{align*}
     \left|\frac{\partial G_L(\vc{y})}{\partial y_i}-\frac{\widehat{\partial G_L}(\vc{y})}{\partial y_i}\right|&=\big|\mathbb{E}_{\vc{y}}[f([\vc{x}]_{+i})]-\mathbb{E}_{\vc{y}}[f([\vc{x}]_{-i})]\displaybreak[0]\\
    &- \mathbb{E}_{\vc{y}}[\hat{f_L}([\vc{x}]_{+i})] + \mathbb{E}_{\vc{y}}[\hat{f_L}([\vc{x}]_{-i})]\big|\\
    &\leq \mathbb{E}_{\vc{y}}[|f([\vc{x}]_{+i})-\hat{f_L}(\vc{[\vc{x}]_{+i}})|]\displaybreak[0]\\
    &+ \mathbb{E}_{\vc{y}}[|f([\vc{x}]_{-i}) - \hat{f_L}([\vc{x}]_{-i})|] \\
    &\stackrel{\mbox{\tiny{(\ref{eq:R_L})}}}{\leq} \mathbb{E}_{\vc{y}}[R_L([\vc{x}]_{+i})]+\mathbb{E}_{\vc{y}}[R_L([\vc{x}]_{-i})]\\
    &=\epsilon_{i, L}(\vc{y}),
\end{align*}
where $\epsilon_{i, L}$ is given by \eqref{eq:epsilonL}. By the uniform convergence \eqref{eq:rl}, $\lim_{L\to\infty}\epsilon_{i, L}(\vc{y})=0$, also uniformly on $\vc{y}\in [0,1]^N$ (as the expectation is a weighted sum, with weights in $[0,1]$).  
Setting $\epsilon_L(\vc{y}) = [\epsilon_{i, L}(\vc{y})]_N \in \reals^N$, we conclude that
\begin{align*}
    \big\|\nabla G(\vc{y}) - \widehat{\nabla G_L}(\vc{y})\big\|_2 \leq \|\epsilon_L(\vc{y})\|_2
\end{align*}
where $\lim_{L \to \infty} \|\epsilon_L(\vc{y})\| = 0$, for all $\vc{y}\in [0,1]^N$.

\section{Proof of Theorem \ref{thm: main}} \label{proof: main}

We begin by proving the following auxiliary lemma:
\begin{lemma} \label{lem: Lipschitz}
$G$ is P-Lipschitz continuous with $P = 2\max_{x \in \mathcal{M}}f(\vc{x})$.
\end{lemma}
\begin{proof}
\begin{align*}
    |G(\mathbf{y})-G(\mathbf{y}')| &= \Big|\sum_{x\in \{0,1\}^{N}} f(\vc{x})\prod_{x_i=1} y_i \prod_{x_i=0}(1-y_i) \displaybreak[0]\\
    &\quad - \sum_{x\in \{0,1\}^{N}} f(\vc{x})\prod_{x_i=1} y_i' \prod_{x_i=0}(1-y_i')\Big|\displaybreak[0]\\
    &= \Big|\sum_{x\in \{0,1\}^{N}} f(\vc{x})\Big(\prod_{x_i=1} y_i \prod_{x_i=0}(1-y_i) \displaybreak[0]\\
    &\quad -\prod_{x_i=1} y_i' \prod_{x_i=0}(1-y_i')\Big)\Big|\displaybreak[0]\\
    &\leq \sum_{x\in \{0,1\}^{N}} |f(\vc{x})| \Big|\prod_{x_i=1} y_i \prod_{x_i=0}(1-y_i) \displaybreak[0]\\
    &\quad -\prod_{x_i=1} y_i' \prod_{x_i=0}(1-y_i')\Big|\displaybreak[0]\\
    &\leq f(\vc{x}) \bigg( \sum_{x\in \{0,1\}^{N}}  \Big| \prod_{x_i=1} y_i \prod_{x_i=0}(1-y_i) \Big| \\
    &\quad + \sum_{x\in \{0,1\}^{N}}  \Big| \prod_{x_i=1} y_i' \prod_{x_i=0}(1-y_i') \Big|\bigg)\displaybreak[0]\\
    &\leq 2\max_{x \in \mathcal{M}} f(\vc{x}).
\end{align*}
\end{proof}

The remainder of the proof follows the proof structure in \cite{bian2017guaranteed}. Let $\vc{m}^* \triangleq (\vc{y}^* \vee \vc{y}) - \vc{y} = (\vc{y}^* - \vc{y}) \vee \vc{0} \geq \vc{0}$, where $\vc{x} \vee \vc{y} \triangleq [\max\{x_i, y_i\}]_i$. Since $\vc{m}^* \leq \vc{y}^*$ and $P(\mathcal{M})$ is down-closed, $\vc{m}^* \in P(\mathcal{M})$. By Asm.~\ref{asmp: mon_sub}, $f$ is monotone. Thus, $G(\vc{y} + \vc{m}^*) = G(\vc{y}^* \vee \vc{y}) \geq G(\vc{y}^*)$. If we define a uni-variate auxiliary function $h_{\vc{y}, \vc{m}}(\xi) \triangleq G(\vc{y} + \xi \vc{m}^*)$, where $\xi \geq 0$, $\frac{dh_{\vc{y}, \vc{m}}(\xi)}{d\xi} = \langle \vc{m}^*, \nabla G(\vc{y} + \xi \vc{m}^*) \rangle$. $h_{\vc{y}, \vc{m}}(\xi)$ is concave because the multilinear relaxation $G$ is concave along non-negative directions due to submodularity of $f$, given by Asm.~\ref{asmp: mon_sub}. Hence,
\begin{equation} \label{eq: univariate}
\begin{split}
    h_{\vc{y}, \vc{m}}(1) - h_{\vc{y}, \vc{m}}(0) &= G(\vc{y} + \vc{m}^*) - G(\vc{y}) \\
    &\leq \frac{dh_{\vc{y}, \vc{m}}(\xi)}{d\xi} \Bigg|_{\xi=0} \times 1 = \langle \vc{m}^*, \nabla G(\vc{y}) \rangle
\end{split}
\end{equation}

For the $k^{th}$ iteration of the continuous greedy algorithm, let $\vc{m}_k \triangleq \argmax_{\vc{m} \in P(\mathcal{M})} \langle \vc{m}, \nabla \widehat{G_L} (\vc{y}_k) \rangle$, $\vc{y}_k \in P(\mathcal{M})$ be the output solution obtained by the algorithm and $\vc{y}^*$ be the optimal solution of (\ref{eq: multilinProb}). Since $\mathbf{y}_k$ is a convex linear combination of the points in $P(\mathcal{M})$, $\vc{y}_k \in P(\mathcal{M})$. Using Thm.~\ref{thm:gradientBias} for $\vc{m} \geq \vc{0}$, due to Asm.~\ref{asmp: f_isInForm}:
\begin{align*}
    &\max_{\vc{m} \in P(\mathcal{M})} \langle \vc{m}, \nabla \widehat{G_L} (\vc{y}_k) \rangle \stackrel{\mbox{\tiny{(\ref{eq:estimator_bound})}}}{\geq}\\
    &\qquad\max_{\vc{m} \in P(\mathcal{M})} (\vc{m}^T \nabla G(\vc{y}_k) - \vc{m}^T \epsilon_L(\vc{y}_k)) \\
    &\quad \geq \max_{\vc{m} \in P(\mathcal{M})} \vc{m}^T \nabla G(\vc{y}_k) - \max_{\vc{m} \in P(\mathcal{M})} \vc{m}^T \epsilon_L(\vc{y}_k) \\
    &\quad \geq \max_{\vc{m} \in P(\mathcal{M})} \vc{m}^T \nabla G(\vc{y}_k) - \max_{\vc{m} \in P(\mathcal{M})}\|\vc{m}\|\,\|\epsilon_L(\vc{y}_k)\|
\end{align*}
due to Cauchy-Schwarz inequality. Replacing $D = \max_{\mathbf{m} \in P(\mathcal{M})} \|\mathbf{m}\|_2$ and $\varepsilon(L) =\max_{k}\| \epsilon_L(\vc{y}_k) \|_2$,
\begin{equation} \label{eq: k-th_step}
\begin{split}
    \langle \vc{m}_k, \nabla \hat{G}_L (\vc{y}_k) \rangle &\geq \langle \vc{m}^*, \nabla G(\vc{y}_k) \rangle - D\,\varepsilon(L)\\
    &\stackrel{\mbox{\tiny{(\ref{eq: univariate})}}}{\geq} G(\vc{y} + \vc{m}^*) - G(\vc{y}) - D\,\varepsilon(L)\\
    &\geq G(\vc{y}^*) - G(\vc{y}_k) -  D\,\varepsilon(L)
\end{split}
\end{equation}

The uni-variate auxiliary function $h_{\mathbf{y}, \mathbf{m}}$ is $P$-Lipschitz since the multilinear realization $G$ is $P$-Lipschitz by Lem.~\ref{lem: Lipschitz}. Then for $h_{\mathbf{y}, \mathbf{m}}(\xi)$ with $P$-Lipschitz continuous derivative in $[0, 1]$ where $(P > 0)$, we have
\begin{equation} \label{eq: LipAssump}
\begin{split}
    -\frac{P}{2}\xi^2 &\leq h_{\mathbf{y}, \mathbf{m}}(\xi) - h_{\mathbf{y}, \mathbf{m}}(0) - \xi \nabla h_{\mathbf{y}, \mathbf{m}}(0) \\
    &= G(\mathbf{y} + \xi \mathbf{m}) - G(\mathbf{y}) - \xi \langle \mathbf{m}, \nabla G(\mathbf{y}) \rangle
\end{split}
\end{equation}
$\forall\xi \in [0, 1]$. Hence the difference between the ${(k+1)}^{th}$ and $k^{th}$ iteration becomes 
\begin{align*}
    &G(\vc{y}_{k+1}) - G(\vc{y}_k) = G(\vc{y}_k +\gamma_k \vc{m}_k) - G(\vc{y}_k)\\
    &\quad = h_{\vc{y}, \vc{m}}(\gamma_k) - h_{\vc{y}, \vc{m}}(0)\\
    &\quad \stackrel{\mbox{\tiny{(\ref{eq: LipAssump})}}}{\geq} \gamma_k \langle \vc{m}, \nabla G(\vc{y}) \rangle - \frac{P}{2} \gamma_k^2 \\
    &\quad \stackrel{\mbox{\tiny{(\ref{eq: k-th_step})}}}{\geq} \gamma_k[G(\mathbf{y}^*) - G(\mathbf{y}_k)] -  \gamma_k D\,\varepsilon(L) - \frac{P}{2} \gamma_k^2
\end{align*}
Rearranging the terms,
\begin{align*}
    G(\mathbf{y}_{k+1}) - G(\mathbf{y}^*) &\geq (1 - \gamma_k) [G(\mathbf{y}_k) - G(\mathbf{y}^*)] \\
    &\quad - \gamma_k D \varepsilon(L) - \frac{P}{2} \gamma_k^2
\end{align*}
If we sum up the inequalities $\forall k = 0, 1, ... \, , K-1$. We get,
\begin{align*}
    G(\mathbf{y}_K) - G(\mathbf{y}^*) &\geq \prod_{k = 0}^{K-1} (1 - \gamma_k) [G(0) - G(\mathbf{y}^*)] \\
    &- D \varepsilon(L) \sum_{k = 0}^{K-1} \gamma_k - \frac{P}{2} \sum_{k = 0}^{K-1} \gamma_k^2
\end{align*}
Knowing that $\sum_{k = 0}^{K-1} \gamma_k = 1$, and $1 - \gamma_k \leq e^{-\gamma_k}$,
\begin{align*}
    G(\mathbf{y}^*) -   G(\mathbf{y}_K) &\leq e^{-\sum_{k = 0}^{K-1} \gamma_k} [G(\mathbf{y}^*) - G(0)] \\
    &+ D \varepsilon(L) + \frac{P}{2} \sum_{k = 0}^{K-1} \gamma_k^2
\end{align*}
Rearranging the terms,
\begin{align} \label{eq: lastStep}
    G(\mathbf{y}_K) \geq \Big(1 - \frac{1}{e}\Big) G(\mathbf{y}^*) - D \varepsilon(L) - \frac{P}{2} \sum_{k = 0}^{K-1} \gamma_k^2 + \frac{1}{e} G(\vc{0})
\end{align}
In order to minimize $\sum_{k = 0}^{K-1} \gamma_k^2$ when $\sum_{k = 0}^{K-1} \gamma_k = 1$, Lagrangian method can be used. Let $\lambda$ be the Lagrangian multiplier, then
\begin{align*}
    \mathcal{L}(\gamma_0, ..., \gamma_{K-1}, \lambda) = \sum_{k = 0}^{K-1} \gamma_k^2 + \lambda \bigg[\sum_{k = 0}^{K-1} \gamma_k - 1\bigg]
\end{align*}
For $\gamma_0 = ... = \gamma_{K-1} = \frac{1}{K}$, $\sum_{k = 0}^{K-1} \gamma_k^2$ reaches its minimum which is $\frac{1}{K}$. Moreover, we have $\vc{y}_0 = \vc{0}$, and hence $G(\vc{y}_0) = 0$. Rewriting (\ref{eq: lastStep}),
\begin{align*}
    G(\mathbf{y}_K) \geq \Big(1 - \frac{1}{e}\Big) G(\mathbf{y}^*) - D \varepsilon(L) - \frac{P}{2K}
\end{align*}


\section{Detailed Comparison to Bound by Mahdian et al.~\cite{mahdian2020kelly}}\label{app:compareBounds}

We start by rewriting the bound provided by Mahdian et. al. \cite{mahdian2020kelly} with our notation. In  App.~C.2 of \cite{mahdian2020kelly}, given a set of continuous functions $\{h_j\}_{j\in\{1,\ldots,M\}}$ where their first $L+1$ derivatives are in $[0,1)$, they give an upper bound on the bias of the polynomial estimator given in (\ref{eq: poly_estimator}) as:
$$\varepsilon(L) \leq \frac{2M W}{(L+1)!},$$
where $W=\max_{j\in\{1,\ldots,M\}, s'\in[0,1)} {h_j^{(L+1)}(s')}$
. This statement holds under the assumption that $W$ is a finite constant, independent of $L$. However, this does not hold for $h(s)=\frac{s}{1-s}$ and $h(s)=\log(1+s)$. In fact, for $h(s)=\frac{s}{1-s}$ and $h(s)=\log(1+s)$, $W$ goes to infinity as $L$ goes to infinity. In contrast, we make no such assumption on the derivatives when providing a bound for the bias $\varepsilon(L)$ (see Appendices~\ref{app: proof_bias_bound}, \ref{app: proof_bias_bound_IM}, and \ref{app:CN}).

\section{Complexity}\label{app:complexity}
The continuous-greedy algorithm described in Alg.~\ref{alg:cont-greed} runs for $K=1/\gamma$ iterations. In each iteration, $\widehat{\nabla G_L}$ is calculated and (\ref{eq:m_k}) is solved with that $\widehat{\nabla G_L}$. The complexity of calculating $\widehat{\nabla G_L}$ is polynomial with the size of the ground set, $N$, with the total number of monomials in (\ref{eq:multi}), $\sum_{j=1}^M \mathcal{I}$, and with the average number of variables appearing in each monomial, $\bar{\mathcal{J}}$. For polymatroids, solving (\ref{eq:m_k}) amounts to solving a linear program, which can also be done in polynomial time that depends on the type of matroid \cite{calinescu2011maximizing}. Specifically for partition matroids however, the solution has a simple water-filing property, and can be obtained $N\log N$ time by sorting the gradient elements corresponding to each partition. Hence, for partition matroids, the entire algorithm takes $O(K(N(\sum_{j=1}^M \mathcal{I})\bar{\mathcal{J}}+m(N\log N + k + m)))$ steps where $m$ is the number of partitions and $k$ is the constraint on each partition.

\section{Proofs of Example Properties}
\subsection{Proof of Theorem \ref{thm: epsilon_bound}.} \label{app: proof_bias_bound}
We begin by characterizing the residual error of the Taylor series of $h(s)=\log(1+s)$ around $1/2$:
\begin{lemma} \label{lem: R_iL_bound}
Let $\hat{h}_{L}(s)$ be the $L^{\text{th}}$ order Taylor approximation of $h(s) = \log(1 + s)$ around $1/2$, given by \eqref{eq: taylor_f_iL}. Then, $\hat{h}$, satisfies the second condition of Asm.~\ref{asmp: f_isInForm}, with residuals:  
\begin{equation}
    R_{j,L}(s) =\frac{1}{(L+1) 2^{L+1}}.
\end{equation}
\end{lemma}
\begin{proof}
By the Lagrange remainder theorem,
\begin{equation*}
\begin{split}
 \left\lvert h_i(s)-\hat{h}_{ L}(s)\right\rvert&=\left\lvert\frac{h_i^{(L+1)}(s')}{(L+1)!} \left(s-\frac{1}{2}\right)^{L+1}\right\rvert \\
 &= \left\rvert \frac{\left(s-{1}/{2}\right)^{L+1}}{(L+1)\left(1+s'\right)^{L+1}} \right\lvert
\end{split}
\end{equation*}
for some $s'$ between $s$ and $1/2$. Since $s \in [0, 1]$, (a) $|s-\frac{1}{2}|\leq \frac{1}{2}$, and (b) $s'\in[0, 1]$.  Hence
 $\left\lvert h_i(s) - \hat{h}_{i, L}(s) \right\rvert \leq \frac{1}{(L+1)2^{L+1}} .$
\end{proof}
To conclude the theorem, observe that:
\begin{align*}
    \epsilon_{i, L}(\vc{y}) &= \mathbb{E}_\vc{y}[R_L([\vc{x}]_{+i})] + \mathbb{E}_\vc{y}[R_L([\vc{x}]_{+i})] \\
    &= 2\mathbb{E}_\vc{y}\left[\textstyle\sum_{i=1}^M |R_{i,L}(s_i)|\right] \\
    &\leq 2\mathbb{E}_\vc{y}\left[\textstyle\sum_{i=1}^M \frac{1}{(L+1) 2^{L+1}}\right]
    = \frac{2M}{(L+1) 2^{L+1}}
\end{align*}
Then, $\varepsilon(L) \leq \frac{M\sqrt{N}}{(L+1) 2^{L}}$. \hspace{\stretch{1}} \qed

\subsection{Proof of Theorem \ref{thm: epsilon_bound_IM}.}\label{app: proof_bias_bound_IM}

To prove the theorem, observe that:
\begin{equation*}
\begin{split}
\epsilon_{i, L}(\vc{y}) &= \mathbb{E}_\vc{y}[R_L([\vc{x}]_{+i})] + \mathbb{E}_\vc{y}[R_L([\vc{x}]_{-i})] \\
&\leq 2\mathbb{E}_\vc{y}\left[\textstyle\sum_{i=1}^{M} \frac{1}{M(L+1) 2^{L+1}}\right]
= \frac{1}{(L+1) 2^{L}}
\end{split}
\end{equation*}
Hence, for all $\vc{y}\in [0,1]^N$, $\varepsilon(L) \leq \frac{\sqrt{N}}{(L+1) 2^{L}}$. \hspace{\stretch{1}} \qed

\section{Example: Cache Networks (CN)\cite{mahdian2020kelly}.} \label{app:CN}
A Kelly cache network can be represented by a graph $G(V, E)$, $|E|=M$, service rates $\mu_j$, $j \in E$, storage capacities $c_v$, $v \in V$, a set of requests $\mathcal{R}$, and arrival rates $\lambda_r$, for $r \in \mathcal{R}$. Each request is characterized by an item $i^r\in\mathcal{C}$ requested, and a path $p^r\subset V$ that the request follows.  For a detailed description of these variables, please refer to \cite{mahdian2020kelly}. Requests are forwarded on a path until they meet a cache storing the requested item. In steady-state,  the traffic load on an edge $(u,v)$ is given by  
\begin{equation} \label{eq:CNgi}
    g_{(u, v)}(\vc{x}) = \frac{1}{\mu_{u, v}}\sum_{r \in \mathcal{R}:(v, u)\in p^r} \lambda^r \prod_{k'=1}^{k_{p^r}(v)}(1-x_{p_k^r, i^r}).
\end{equation}
where $\vc{x}\in \{0,1\}^{|V||\mathcal{C}|}$ is a vector of binary coordinates $x_{vi}$ indicating if $i\in \mathcal{C}$ is stored in node $v\in V$. If $s$ is the load on an edge,  the expected total number of packets in the system is given by $h(s)=\frac{s}{1-s}$.
Then using the notation $j=(u, v) \in E$ to index edges, the expected total number of packets in the system in steady state can indeed be written as $\sum_{j=1}^M h_j(g_j(\vc{x}))$ \cite{mahdian2020kelly}. 
Mahdian et al.~maximize the  \emph{caching gain} $f: \{0, 1\}^{|V||\mathcal{C}|} \rightarrow \reals_+$ as 
\begin{equation} \label{eq:CN}
    f(\vc{x}) = \textstyle\sum_{j=1}^M h_j(g_j(\vc{0})) - \sum_{j=1}^M h_j(g_j(\vc{x}))
\end{equation}
subject to the capacity constraints in each class.
The caching gain $f(\vc{x})$ is  monotone and submodular, and the capacity constraints form a partition matroid \cite{mahdian2020kelly}. 
Moreover, $h(s)=\frac{s}{1-s}$ can be approximated within arbitrary accuracy by its $L^{\text{th}}$-order Taylor approximation around $0$, given by:
\begin{equation} \label{eq: f_iL_CN}
    \hat{h}_{ L}(s) = \textstyle\sum_{\ell = 1}^L s^\ell
\end{equation}
We show in the following lemma 
that this estimator ensures that $f$ indeed satisfies Ass.~\ref{asmp: f_isInForm}: 
\begin{lemma} \label{lem:bound_CN}
Let $\hat{h}_{j,L}(s)$ be the $L^{th}$ Taylor polynomial of $h_j(s)=\frac{s}{1-s}$ around $0$. Then, $h_j(s)$ and its polynomial estimator of degree $L$, $\hat{h}_{L}(s)$, satisfy Asm.~\ref{asmp: f_isInForm} where 
\begin{equation}
    R_{j,L}(s) \leq \frac{\bar{s}^{L+1}}{1 - \bar{s}}.
\end{equation}
\end{lemma}
$L^{th}$ Taylor polynomial of $h_i(s)$ around $0$ is
\begin{equation}
    \hat{h}_{L}(s) = \textstyle\sum_{l = 0}^L \frac{h_i^{(\ell)}(0)}{\ell!} s^{\ell} = \sum_{\ell=1}^{L} s^{\ell}
\end{equation}
where $h_i^{(\ell)}(s) = \frac{\ell!}{(1-s)^{\ell+1}}$ for $f_i(s) = \frac{s}{1-s}$. 
\begin{align*}
    h_i(s) &= \frac{s}{1 - s} = \textstyle\sum_{\ell=1}^{\infty} s^{\ell} = \sum_{\ell=1}^{L} s^{\ell} + \sum_{\ell=L+1}^{\infty} s^{\ell} \\
    &= \textstyle\sum_{\ell=1}^{L}s^{\ell}+s^L\sum_{\ell=1}^{\infty}s^{\ell}=\sum_{\ell=1}^{L} s^{\ell} + \frac{s^{L+1}}{1-s}
\end{align*}
Then, the bias of the Taylor Series Estimation around $0$ becomes:
 \begin{align*}
    \left| \frac{s}{1 - s} - \textstyle\sum_{n=1}^{L} s^n \right|  =  \frac{s^{L+1}}{1-s} \leq \frac{\bar{s}^{L+1}}{1 - \bar{s}} = R_{i,L}(s).
\end{align*}
for all $s \in [0, \bar{s}]$ where $\bar{s} = \max_{i \in M} s_i$. \hspace{\stretch{1}} \qed
Furthermore, we  bound the estimator bias appearing in Thm.~\ref{thm: main} as follows:
\begin{theorem} \label{thm:epsilon_bound_CN}
Assume a caching gain function $f:~\{0,1\}^{|V||\mathcal{C}|} \rightarrow \reals_+$ that is given by (\ref{eq:CN}). Then, consider Algorithm \ref{alg:cont-greed} in which $\nabla G(\mathbf{y}_K)$ is estimated via the polynomial estimator given in (\ref{eq: poly_estimator}) where $\hat{f}_{L}(\vc{x})$ is the $L^{th}$ Taylor polynomial of $f(\vc{x})$ around $0$. Then, the bias of the estimator is bounded by
\begin{equation}
    \varepsilon(L) \leq 2M\sqrt{{|V||\mathcal{C}|}}\frac{\bar{s}^{L+1}}{1-\bar{s}},
\end{equation}
where $\bar{s}<1$ is the largest load among all edges when caches are empty.  
\end{theorem}
\begin{proof}
Since $\lim_{L\to \infty} \frac{\bar{s}^{L+1}}{1 - \bar{s}} = 0$, for all $\bar{s} \in [0, 1)$, Taylor approximation gives an approximation guarantee for maximizing the queue size function by Asm.~\ref{asmp: f_isInForm},  where the error of the approximation is given by Thm.~\ref{thm:gradientBias} as
\begin{align*}
\epsilon_{i, L}(\vc{y}) &= 2\mathbb{E}_\vc{y}[R_L([\vc{x}]_{+i})] + \mathbb{E}_\vc{y}[R_L([\vc{x}]_{+i})]\\
&= \mathbb{E}_\vc{y}\left[\textstyle\sum_{i=1}^M |R_{i,L}(s_i)|\right]\\
&\leq 2\mathbb{E}_\vc{y}\left[\textstyle\sum_{i=1}^M \frac{\bar{s}^{L+1}}{1 - \bar{s}}\right] = 2M\frac{\bar{s}^{L+1}}{1 - \bar{s}}
\end{align*} Then, $\varepsilon(L) \leq 2M\sqrt{{|V||\mathcal{C}|}}\frac{\bar{s}^{L+1}}{1-\bar{s}}$.
\end{proof}

\section{Pipage Rounding via Taylor Estimator}\label{app:pipage_poly} As explained, each step of pipage rounding requires evaluating the multilinear relaxation $G(\hat{\vc{y}}),$ which is generally infeasible and is usually computed via the time-consuming sampling estimator (see Sec.~\ref{sec:CG}).  Here we show that these evaluations can be alternatively done via the polynomial estimator, while having theoretical guarantees. First note that similar to the case of gradients in Thm.~\ref{thm:gradientBias} the difference between $G$ and the  multilinear relaxation of  polynomial estimator $\hat{G}(\vc{y}) \triangleq  \mathbb{E}_{\vc{x}\sim \vc{y}}[{\hat{f}_L(\vc{x}})] = \hat{f}_L(\vc{y})$ is bounded:
\begin{align}\label{eq:bound}
    |G(\vc{y}) - \hat{G}(\vc{y})| 
    &\leq \mathbb{E}_{\vc{x}\sim \vc{y}} [R_L(\vc{x})]
    \leq \bar{R}_L,
\end{align}
where $\bar{R}_L\triangleq \max_{\vc{y}\in P(\mathcal{M})}  \mathbb{E}_{\vc{x}\sim \vc{y}} [R_L(\vc{x})]$. Again similar to the proof in App.~\ref{proof: gradientBias} and due to the uniform convergence in \eqref{eq:rl} it holds that that $\lim_{L\to \infty} \bar{R}_L =0.$ 
Now we can show our main result on pipage rounding via our polynomial estimator. 
\begin{theorem}
Given a fractional solution $\vc{y}\in P(\mathcal{M})$ the pipage rounding method in which the polynomial estimator $\hat{G}$ is used instead of $G$ terminates in $O(N)$ rounds and  the obtained solution $\hat{\vc{y}}\in \domain$ satisfies the following
\begin{align*}
    G(\hat{\vc{y}})\geq G(\vc{y}) - 2(N+1) \bar{R}_L.
\end{align*}
\end{theorem}
\begin{proof}
At round $k$, given a solution $\vc{y}^{(k)}\in P(\mathcal{M})$ due to the properties of the multilinear  relaxation there exists a point $\hat{\vc{y}}^{(k)}$, s.t., (a) $G(\hat{\vc{y}}^{(k)}) \geq G(\vc{y}^{(k)})$ and (b) $\hat{\vc{y}}^{(k)}$ has at least one less fractional element, i.e., 
$\{j\in\{1,\ldots,N\}\,|\,\vc{y}^{(k)}_j\in\{0,1\} \}\subset
\{j\in\{1,\ldots,N\}\,|\,\hat{\vc{y}}^{(k)}_j\in\{0,1\} \}$ \cite{ageev2004pipage}. From \eqref{eq:bound} and (a) we have the following:
\begin{align}\label{eq:rnd}
    \hat{G}(\hat{\vc{y}}^{(k)})\stackrel{\mbox{\tiny{(\ref{eq:bound})}}}{\geq} G(\hat{\vc{y}}^{(k)}) - \bar{R}_L 
    &\stackrel{\mbox{\tiny{(a)}}}{\geq} G(\vc{y}^{(k)}) - \bar{R}_L \nonumber\\&
    \stackrel{\mbox{\tiny{(\ref{eq:bound})}}}{\geq} \hat{G}(\vc{y}^{(k)}) - 2\bar{R}_L,
\end{align}
in other words the estimated objective at $\hat{\vc{y}}^{(k)}$ is at most $2\bar{R}_L$ worse than the estimated value at $\vc{y}^{(k)}.$ Now  given input to pipage rounding as $\vc{y}^{(0)}=\vc{y}$ and at each round setting $\vc{y}^{(k+1)}=\hat{\vc{y}}^{(k)}$ from \eqref{eq:rnd} we have that:
\begin{align}\label{eq:round_telescope}
\hat{G}(\vc{y}^{(k)})\geq \hat{G}(\vc{y}^{(0)}) - 2k \bar{R}_L \stackrel{\mbox{\tiny{(\ref{eq:bound})}}}{\geq} G(\vc{y}^{(0)}) - 2k\bar{R}_L - \bar{R}_L.
\end{align}
Furthermore, from (b) it follows that this process ends  at $k^*\leq N$ rounds as $\vc{y}^{(0)}$ has at most $N$ fractional elements. Plus, for the final solution $\hat{\vc{y}}=\vc{y}^{(k^*)}$ it holds that:
\begin{align*}
   G(\hat{\vc{y}}) \stackrel{\mbox{\tiny{(\ref{eq:bound})}}}{\geq} \hat{G}(\hat{\vc{y}}) -  \bar{R}_L
   \stackrel{\mbox{\tiny{(\ref{eq:round_telescope})}}}{\geq}&
   G(\vc{y}) - 2(k^*+1)\bar{R}_L\\ \geq &G(\vc{y}) - 2(N+1)\bar{R}_L.
\end{align*}
\end{proof}

\end{document}